\documentclass[journal,twocolumn,10pt]{IEEEtran}
\usepackage{amsmath,epsfig} 

\usepackage{floatflt} 
\usepackage{paralist} 
\usepackage{algorithm} 
\usepackage{graphicx}
\usepackage{amsthm}
\usepackage{amssymb}
\usepackage{listings}    
\usepackage{xcolor}
\usepackage{hyperref}
\usepackage{indentfirst}
\usepackage{booktabs}
\usepackage{subfigure}
\usepackage{caption}
\usepackage{calligra}
\usepackage{bm}
\usepackage{fancyhdr}
\usepackage{float}
\usepackage{pifont}
\usepackage{colortbl} 
\usepackage{amsmath}
\usepackage{MnSymbol}%
\usepackage{amsbsy}

\theoremstyle{definition}

\newcommand{\mypar}[1]{{\bf #1.}}
\newtheorem{myLem}{Lemma}
\newtheorem{myAlg}{Algorithm}
\newtheorem{myThm}{Theorem}

\newcommand{\R}{\ensuremath{\mathbb{R}}}

\DeclareMathOperator{\Id}{I}

\def\a{\mathbf{a}}

\def\x{\mathbf{x}}
\def\y{\mathbf{y}}
\def\h{\mathbf{h}}

\def\n{\mathbf{n}}

\def\q{\mathbf{q}}

\def\s{\mathbf{s}}

\def\tt{\mathbf{t}}
\def\vv{\mathbf{v}}

\def\z{\mathbf{z}}

\def\N{\mathcal{N}}
\def\V{\mathcal{V}}

\def\M{\mathcal{M}}
\def\U{\mathcal{U}}

\DeclareMathOperator{\Adj}{A}

\DeclareMathOperator{\HH}{H}

\DeclareMathOperator{\X}{X}

\DeclareMathOperator{\Um}{U}

\DeclareMathOperator{\W}{W}
\DeclareMathOperator{\Ss}{S}

\newcommand{\tabincell}[2]
{
\begin{tabular}{@{}#1@{}}#2\end{tabular}
}

\title{Sampling and Recovery of Graph Signals based on Graph Neural Networks}
\author{Siheng Chen$^{*}$~\IEEEmembership{ Member,~IEEE}, 
        Maosen Li$^{*}$~\IEEEmembership{Student Member,~IEEE},
        Ya Zhang~\IEEEmembership{ Member,~IEEE},
\thanks{S. Chen is with Mitsubishi Electric Research Laboratories (MERL), Cambridge, MA, USA. Email: schen@merl.com. M. Li and Y. Zhang are with 
 Cooperative Medianet Innovation Center at Shanghai Jiao Tong University, Shanghai, China. Emails: maosen\_li@sjtu.edu.cn, ya\_zhang@sjtu.edu.cn. $*$ Siheng Chen and Maosen Li contributed equally to this work.
 }}
\begin{document}
\maketitle

\begin{abstract}
We propose interpretable graph neural networks for sampling and recovery of graph signals, respectively. To take informative measurements, we propose a new graph neural sampling module, which aims to select those vertices that maximally express their corresponding neighborhoods. Such expressiveness can be quantified by the mutual information between vertices' features and neighborhoods' features, which are estimated via a graph neural network. To reconstruct an original graph signal from the sampled measurements, we propose a graph neural recovery module based on the algorithm-unrolling technique, which transforms each iteration of an analytical recovery algorithm to a network layer. Compared to previous analytical sampling and recovery, the proposed methods are able to flexibly learn a variety of graph signal models from data by leveraging the learning ability of neural networks; compared to previous neural-network-based sampling and recovery, the proposed methods are designed through exploiting specific graph properties and provide interpretability. We further design a new multiscale graph  neural network, which is a trainable multiscale graph filter bank and can handle various graph-related learning tasks. The multiscale network leverages the proposed graph neural sampling and recovery modules to achieve multiscale representations of a graph, and involves a new feature-crossing layer, which allows intermediate features from different scales to communicate and further improves the learning ability. Due to the iconic feature-crossing layers, we call this architecture a graph cross network.  In the experiments, we illustrate the effects of the proposed graph neural sampling and recovery modules and find that the modules can flexibly adapt to various graph structures and graph signals. In the task of active-sampling-based semi-supervised learning, the graph neural sampling module improves the classification accuracy over $10\%$ in Cora dataset. We further validate the proposed multiscale graph neural network on several standard datasets for both vertex and graph classification. The results show that our method consistently improves the classification accuracies.
\end{abstract}



\section{Introduction}
\label{sec:intro}
With the explosive growth of information and communication, data today is generated at an unprecedented rate from various sources, including social, biological and physical infrastructure~\cite{Jackson:08,Newman:10}, among others. Unlike time-series signals or images, these signals possess complex, irregular structures, which can be modeled as graphs. Analyzing graph signals requires new concepts and tools to handle the underlying irregular relationships, leading to the emerging fields of graph signal processing~\cite{OrtegaFKMV:18} and graph neural networks~\cite{BronsteinBLSV:17}. Graph signal processing generalizes the classical signal processing toolbox to the graph domain and provides a series of techniques to process graph signals~\cite{ShumanNFOV:13,SandryhailaM:14}, including graph-based transformations~\cite{HammondVG:11}, graph filter bank design~\cite{NarangO:12,ShumanFV:16} and graph topology learning~\cite{DongTRF:19}. On the other hand, graph neural networks expand deep learning techniques to the graph domain and provide a data-driven framework to learn from graph signals with graph structures as induced biases~\cite{BronsteinBLSV:17}. For example, graph convolutional networks and the variants have attained remarkable success in social network analysis~\cite{KipfW:17}, point cloud processing~\cite{WangSLSBS:19}, action recognition~\cite{LiCCZWT:19} and relational inference~\cite{HuCZG:20}. In this work, we consider sampling and recovery of graph signals, which has been highly concerned in both graph signal processing~\cite{ChenVSK:15,AnisAO:15,BaiWCNG:20} and graph neural networks~\cite{YingYMRHL:18,LeeLK:19}. 


In classical signal processing, signal sampling and recovery are key techniques to link continuous-time signals (functions of a real variable) and discrete-time signals (sequences indexed by integers)~\cite{VetterliKG:12}. Sampling produces a sequence from a function, and recovery produces a function from a sequence. In the graph domain, sampling is a reduction of a graph signal to a small number of measurements, and recovery is a reconstruction of a graph signal from noisy, missing, or corrupted measurements. The recovery techniques are usually related to the sampling procedure because the sampling procedure determines the property of the measurements. Several types of graph signal sampling are considered in literature. For example, subsampling selects one vertex in each measurement~\cite{ChenVSK:15,AnisAO:15}; local-neighboring sampling selects a set of connected vertices in each measurement~\cite{WangCG:16}; and aggregated sampling considers a time-evolving graph signal and selects one vertex at different time stamps in each measurement~\cite{MarquesSGR:15}. We focus on subsampling in this paper.

Based on the graph signal processing framework,  analytical graph signal sampling and recovery operators are designed by mathematics and provide  optimal or near-optimal solutions for some specific graph signal model, which is a class of graph signals with certain graph-related properties. For example,~\cite{ChenVSK:15,AnisCO:16} focus on a bandlimited graph signal model and presents a graph counterpart of Shannon–Nyquist sampling theory. Recently,~\cite{TanakaE:20} considers a periodic graph spectrum subspace and provides a generalized sampling framework. The techniques related to analytical graph signal sampling and recovery have been broadly applied to multiscale graph filter bank design~\cite{ShumanFV:16}, sensor placement~\cite{SakiyamaTTO:16}, semi-supervised learning~\cite{GaddeAO:14}, traffic monitoring~\cite{ChenYFK:16}, 3D point cloud processing~\cite{ChenTFVK:18}, and many others.  

A fundamental challenge for analytical sampling and recovery is that we may not know an appropriate graph signal model in practice. When a mathematically-designed graph signal model is far away from the  ground-truth, yet unknown graph signal model that  generates real-world data, it might cause a significant performance gap between theory and practice. Furthermore, a ground-truth graph signal model could be too complicated to be precisely described in mathematics, leading to computationally intensive algorithms. To solve this issue, it is desirable to achieve a data-driven model that is appropriately learnt from given graph signals. In other words, sampling and recovery strategies should have sufficient flexibility to learn from and adapt to arbitrary graph signal models. 

On the other hand, from a perspective of deep neural networks, sampling, or called pooling, is also critical in many learning tasks as it can adjust data resolution and enable multiscale analysis, which has achieved significant success in image and video analysis~\cite{RonnebergerFB:15}. In the graph domain, graph sampling/pooling is also a key component for a graph neural network. Compared to sampling on regular lattices, sampling of graph signals is technically more challenging due to highly irregular structures. To handle this, most previous methods directly train a sampling operator purely based on downstream tasks without explicitly exploring graph-related property~\cite{YingYMRHL:18,LeeLK:19}. It is thus hard to explain the vertex-selection strategy of the learnt sampling operator, which potentially impedes us from improving the method.

In this work, we propose new graph-neural-network-based sampling and recovery operators, which leverage the learning ability of graph neural networks and provide interpretability from a signal processing perspective. To achieve sampling, our strategy is to select those vertices that can maximally express their corresponding neighborhoods. Those vertices can be regarded as the centers of their corresponding communities. We can mathematically quantify the expressiveness by the mutual information between vertex and neighborhood features. Through neural estimation of mutual information, the proposed graph neural sampling module can be optimized and naturally provide an affinity metric between a vertex and a neighborhood in a data-driven fashion. By varying the neighborhood radius, we are allowed to tradeoff between local and global mutual information. Compared to analytical graph signal sampling~\cite{ChenVSK:15,AnisAO:15,TanakaE:20}, the proposed graph neural sampling module is model-free and sufficiently flexible to adapt to arbitrary graph signal models and arbitrary subsequent tasks. Compared to previous graph-neural-network-based sampling~\cite{YingYMRHL:18,LeeLK:19}, the proposed one explicitly exploits the graph-related property and provides interpretability. Moreover, previous methods need a subsequent task to provide additional supervision, while the proposed method is rooted in estimating the dependencies between vertices and neighborhoods, which can be trained in either an unsupervised or a supervised setting.

To achieve recovery, we propose a new graph neural recovery module based on algorithm unrolling. We first propose an iterative recovery algorithm and then transform it to a graph neural network by mapping each iteration to a network layer. Compared to analytical graph recovery, the proposed trainable method is able to learn a variety of graph signal models from given graph signals by leveraging neural networks. Compared to many other neural networks~\cite{KipfW:17,VelivckovicCCRLB:18}, the proposed method is interpretable by following analytical iterative steps. 

Based on the proposed graph neural sampling and recovery modules, we further propose a new multiscale graph neural network, which is a trainable counterpart of a multiscale graph filter bank. The proposed network includes three key basic operations: graph downsampling\footnote{Graph downsampling includes two parts: graph signal downsampling and graph structure downsampling. Graph signal sampling is equivalent to graph signal downsampling.}, graph upsampling and graph filtering. Each operations is trainable and extended from its analytical counterpart. Here we implement graph downsampling and graph upsampling by the proposed graph neural sampling and recovery modules, respectively. By adjusting the last layer  and the final supervision, the proposed multiscale graph neural network can handle various graph-related tasks. Compared to a conventional multiscale graph filter bank~\cite{ShumanFV:16}, the proposed  multiscale graph neural network is trainable and more flexible to adapt to a new task. Compared to previous multiscale graph neural networks~\cite{GaoJ:19,DengZWZF:20}, our model leverages the proposed novel graph neural sampling and recovery modules and presents a new feature-crossing layer that allows multiscale features to communicate in the intermediate network layers, and further improve the learning ability. Because of the structure of the feature-crossing layer, we also call our model~\emph{graph cross network} (GXN).

To illustrate the sampling strategy of the proposed graph neural sampling module, we compare the selected vertices based on various sampling methods for both simulated and real-world scenarios. We find that (i) when dealing with smooth graph signals, the proposed graph neural sampling module has similar performance with the analytical sampling; (ii) the proposed graph neural sampling and recovery modules can flexibly adapt to various graph signals, such as piecewise-smooth graph signals. In a task of active-sampling-based semi-supervised learning, the proposed graph neural sampling module improves the classification accuracies of previous sampling methods over $10\%$ in Cora dataset.

To validate the performance of the proposed multiscale graph neural network, we conduct extensive experiments on several standard datasets for both vertex classification and graph classification. Compared to state-of-the-art methods for these two tasks, the proposed method improves the average classification accuracies by $1.15\%$ and $1.30\%$, respectively.

\mypar{Contributions}  The main contributions of the paper include:
\begin{itemize}
\item We propose a novel graph neural sampling module, which is designed and optimized through estimating the dependency between vertex and neighborhood features.

\item We propose a novel graph neural recovery module by unrolling an analytical recovery algorithm, leading to a trainable and interpretable network.

\item We propose novel multiscale graph neural networks, which is a trainable counterpart of multiscale graph filter banks by leveraging the proposed  graph neural sampling and recovery modules.

\item We conduct experiments to illustrate the sampling strategy of the proposed graph neural sampling module and the recovery performance of the proposed graph neural recovery module.

\item We validate the proposed multiscale graph neural network on two tasks: graph classification and vertex classification.  Compared to state-of-the-art methods for these two tasks, the proposed method improves the average classification accuracies by $1.15\%$ and $1.30\%$, respectively.
\end{itemize}

The rest of the paper is organized as follows: Section~\ref{sec:formulation} formulates the task of sampling and recovery of graph signals. Sections~\ref{sec:sampling} and~\ref{sec:recovery} propose a graph neural sampling module and a  graph neural recovery module, respectively. Based on these two modules, Section~\ref{sec:MGNN} further proposes a multiscale graph neural network, which is a trainable counterpart of a multiscale graph filter bank. We illustrate the effects of the proposed graph neural sampling and recovery modules on several exemplar graphs in Section~\ref{sec:illustration}. Finally, the experiments validating the advantages of the proposed multiscale graph neural network are provided in Section~\ref{sec:application}.

\section{Problem Formulation}
\label{sec:formulation}
In this section, we formulate the task of sampling and recovery of graph signals based on two approaches: an analytical approach and a neural-network-based approach. The analytical approach overviews the conventional methods in graph signal processing and designs sampling and recovery operators by mathematics. The neural-network-based approach designs trainable sampling and recovery operators based on deep learning techniques, which lays a foundation for the proposed sampling and recovery methods.

\subsection{Basic concepts}
We consider a graph $G = (\V, \Adj)$, where $\V = \{v_n\}_{n=1}^{N}$ is the set of \emph{vertices} and $\Adj \in \R^{N \times N}$ is the graph adjacency matrix, where the element $\Adj_{i,j}$ represents an weighted edge from the $i$th to the $j$th vertex to characterize the corresponding vertex relation, such as similarity or dependency. Based on the graph $G$, a \emph{graph signal} is defined as a map that assigns a signal coefficient $x_n \in \R$ to the vertex $v_n$; that is, all the vertices carry their associated signal coefficients to form the global signal on the whole graph. A graph signal can be formulated as a length-$N$ vector defined by
$
\x \ = \ \begin{bmatrix}
 x_1 & x_2 & \ldots & x_{N}
\end{bmatrix}^{\top},
$
where the $n$th vector element $x_n$ is indexed by the vertex $v_n$. 

The multiplication between the graph adjacency matrix and a graph signal, $\Adj\x$, replaces the signal coefficient at each vertex with the weighted linear combination of the signal coefficients of the corresponding neighbors according to the relations represented by $\Adj$. In other words, the graph adjacency matrix enables the value at each vertex shift to its neighbors; we thus call it a~\emph{graph shift operator}~\cite{SandryhailaM:13}.

\subsection{Analytical sampling \& recovery framework}
The dual process of sampling and recovery for a graph signal is to use a few measurements to represent a complete graph signal, where sampling is a reduction of a graph signal to a small number of measurements, and recovery is a reconstruction of a graph signal from noisy, missing, or corrupted measurements. There are many types of sampling considered in literature, including subsampling~\cite{ChenVSK:15,AnisAO:15}, neighborhood sampling~\cite{WangCG:16} and aggregated sampling~\cite{MarquesSGR:15}. Here we focus on subsampling, which samples the signal coefficient at one vertex in each measurement.

Suppose that we first sample $M$ coefficients of a graph signal $\x \in \R^N$ to produce a sampled signal $\y \in \R^M$ ($M < N$), where $\y_i = \x_{\M_i}$ with the $i$th sampled index $\M_i \in \{0, 1, \cdots, N-1 \}$. Then, we interpolate the sampled signal $\y$ to obtain the reconstruction $\x' \in \R^N$, which recovers the original graph signal $\x$ either exactly or approximately. Mathematically, the whole process of graph sampling and recovery is formulated as
\begin{eqnarray}
{\rm sampling:}~~&& \y =  \Psi \x,
\nonumber \\ \nonumber
{\rm recovery:}~~&&\x' =  {\Phi} \y = \Phi \Psi \x,
\end{eqnarray}
where the sampling operator $\Psi$ is an $M \times N$ matrix whose $i,j$th element is
\begin{equation}
\label{eq:Psi}
 \Psi_{i,j} = 
  \left\{ 
    \begin{array}{rl}
      1, & j = \M_i;\\
      0, & \mbox{otherwise};
  \end{array} \right. 
\end{equation}
and the recovery operator $\Phi$: {$\R^M \to \R^N$} maps the measurements to a reconstructed graph signal. In general, it is difficult to perfectly recover the original graph signal $\x$ by using its reconstruction $\x'$ due to the information loss during sampling.

We usually consider that an original graph signal $\x$ is generated from a known and fixed graph signal model $\mathcal{X}$, which is a class of graph signals with specific properties that are related to the irregular graph structure $G$. For example, the bandlimited and approximately bandlimited graph signal model are commonly used to model smooth graph signals~\cite{ChenVSK:15,ChenVSK:15c}. A main challenge of designing a pair of sampling and recovery operators is to cope with the graph signal model $\mathcal{X}$. A common prototype is to solve the following optimization problem:
\begin{equation}
\label{eq:analytical_samping_recovery_opt}
  \min_{\Psi, \Phi}~~ \mathbb{E}_{\x \in \mathcal{X}}~ \left\| \x - \Phi \y \right\|_2^2,~~
  {\rm subject~to~~} \y = \Psi \x,
\end{equation}
where the expectation $\mathbb{E}$ can be replaced by other aggregation functions, such as the maximum.  The analytical solutions of~\eqref{eq:analytical_samping_recovery_opt} provide a pair of sampling and recovery operators that theoretically minimize the expected recovery error for the graph signals in the graph signal model $\mathcal{X}$. When $\mathcal{X}$ is the bandlimited graph signal model, we can design the analytical sampling and recovery operators to guarantee the perfect recovery~\cite{ChenVSK:15}, which expands the classical Shannon-Nyquist sampling theory to the irregular graph domain~\cite{Eldar:15}.

With additional constraints, we could consider many variants of~\eqref{eq:analytical_samping_recovery_opt}. For example, we could consider either deterministic or randomized methods to obtain a sampling operator~\cite{ChenVSK:15a,PuyTGV:15}. We could also introduce additional graph-regularization terms to guide the  recovery process, such as the quadratic form of graph Laplacian~\cite{ShumanNFOV:13} and the graph total variantion~\cite{ChenSMK:14}; see more details in recent review papers~\cite{LorenzoBB:20,TanakaEOC:20}.

\begin{figure}[t]
\centering
\includegraphics[width=0.48\textwidth]{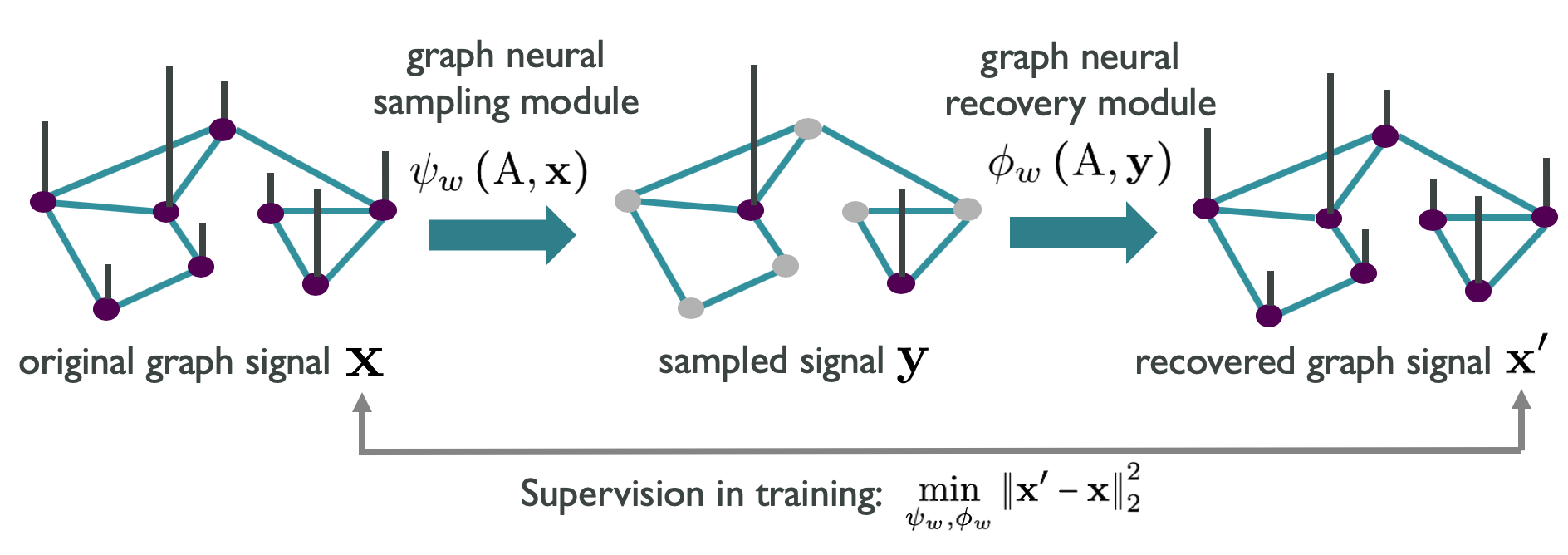}
\caption{\small Graph neural sampling and recovery framework. 
In the training phase, based on the reconstruction of given graph signals, graph neural sampling and recovery modules are optimized through neural networks. After training, we fix both modules and use them to sample and recovery a new graph signal.}
\label{fig:pooling}
\end{figure}

\subsection{Neural sampling \& recovery framework}
A fundamental limitation of the above-mentioned analytical framework is that we may not know the graph signal model in practice. It is then hard to either appropriately choose a graph signal model or precisely design a new mathematical model. To solve this issue, it is desirable to learn an appropriate model in a data-driven manner. Suppose that we are given a set of graph signals, which sufficiently represents a graph signal model that we are interested in. We then aim to train a pair of neural-network-based sampling and recovery operators according to the given data. Those operators then can leverage the learning ability of neural networks to adapt to an arbitrary graph signal model.

Let $\psi_w (\cdot,\cdot)$ and $\phi_w (\cdot, \cdot)$ be graph neural sampling and recovery modules, respectively, where the subscript $w$ generally reflects both functions involve training parameters. We then have
\begin{subequations}
\label{eq:neural_sampling_recovery}
\begin{eqnarray}
\label{eq:neural_sampling}
{\rm neural~sampling:}~&& \y = \Psi \x =  \psi_w \left( \Adj,  \x  \right),
\\ 
\label{eq:neural_recovery}
{\rm neural~recovery:}~&&\x' =  {\phi}_w \left( \Adj, \y \right),
\end{eqnarray}
\end{subequations}
where $\Adj$ is the graph adjacency matrix, the $M \times N$ matrix $\Psi$ follows the definition of the sampling operator~\eqref{eq:Psi}. The graph neural sampling module is supposed to aggregate information from both graph structure and graph signals to determine the sampling operator $\Psi$; and the graph neural recovery module is supposed to exploit information from both graph structure and measurements to construct a graph signal. 

Instead of defining an explicit graph signal model mathematically, we want the network to learn a graph signal model from given data. This is a practical assumption because  graph structures and graph signals are provided together in many scenarios. For example, in a social network, the users' relationship forms a graph structure and users' profile information forms graph signals. Let 
\begin{equation}
\label{eq:graph_signals}
\X = 
\begin{bmatrix} 
\x^{(1)} & \x^{(2)} & \cdots & \x^{(L)} 
\end{bmatrix} 
= \begin{bmatrix} 
\s_1^T \\ \s_2^T \\ \vdots \\ \s_N^T
\end{bmatrix} 
\end{equation}
be a $N \times L$ matrix that contains $L$ graph signals generated from an unknown  graph signal model $\mathcal{X}$. The $i$th column vector $\x^{(i)} \in \R^N$ is the $i$th graph signal and the $v$th row vector $\s_v  \in \R^L$ collects the signal coefficients supported at vertex $v$.

Based on the given data, we consider the following optimization problem 
\begin{eqnarray}
\label{eq:neural_sampling_recovery_opt}
&& \min_{\psi_w (\cdot,\cdot), \phi_w (\cdot,\cdot)} ~~~\sum_{i=1}^L \left\| \x^{(i)} - {\phi}_w \left( \Adj, \y^{(i)} \right) \right\|_2^2
\\ \nonumber
&& {\rm subject~to~~~~}  \y^{(i)} =  \psi_w \left( \Adj, \x^{(i)} \right).
\end{eqnarray}
Comparing~\eqref{eq:analytical_samping_recovery_opt}~and~\eqref{eq:neural_sampling_recovery_opt}, there are two major differences: i) we work with $L$ graph signals, which can be regarded as a proxy of an unknown graph signal model; and ii) the sampling and recovery operators are substituted by neural-network-based modules, which are proposed to implicitly capture an appropriate graph signal model from given graph signals. In~\eqref{eq:analytical_samping_recovery_opt}, $\Psi$ and $\Phi$ are solved analytically; while in~\eqref{eq:neural_sampling_recovery_opt}, $\psi_w (\cdot,\cdot)$ and $\phi_w (\cdot,\cdot)$ are obtained through training.

We can solve the optimization problem by stochastic gradient descent~\cite{Goodfellow:2016}. After training, we fix the trainable parameters in $\psi_w (\cdot,\cdot), \phi_w (\cdot,\cdot)$. Then, for any arbitrary  graph signal $\x$ generated from the same model $\mathcal{X}$, we can follow~\eqref{eq:neural_sampling_recovery}: using $\psi_w (\cdot,\cdot)$ to take measurements and using $ \phi_w (\cdot,\cdot)$ to recover the complete graph signal $\x$.

Based on this new trainable framework, we are going to specifically design a graph neural sampling module and a graph neural recovery module in the following two sections, respectively.

\section{Graph Neural Sampling Module}
\label{sec:sampling}
In this section, we are going to design a neural network $\psi_w (\cdot,\cdot)$ to select a vertices set $\M \subset \mathcal{V}$ that contains $|\M| = M$ vertices. Two key challenges include: i) we need to explicitly describe the importance of a vertex in a complex and irregular graph; and ii) we need to consider information from two sources: graph structures and given graph signals. A trivial neural network architecture with brute-force end-to-end learning can hardly work in this case. We need to exploit graph-related properties. Our main intuition is to select those vertices that can maximally express their corresponding neighborhoods. Those vertices can be regarded as the centers of their corresponding communities. We can mathematically measure the expressiveness by the mutual information between vertices' features and neighborhoods' features. Those features include information from both graph structures and graph signals. The proposed graph neural sampling module can be optimized through estimating this mutual information.

\begin{figure*}[t]
\centering
\includegraphics[width=1\textwidth]{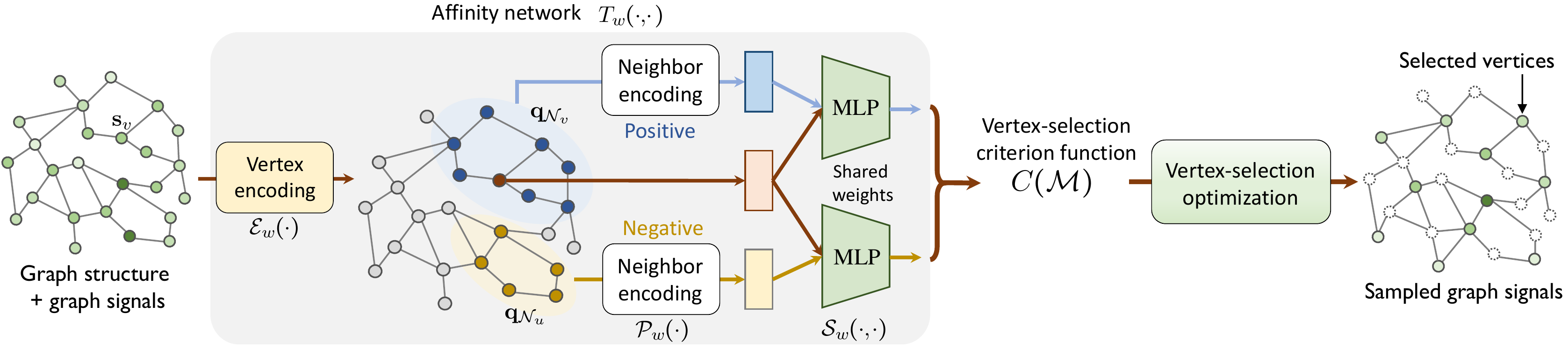}
\caption{\small Implementation of graph neural sampling module. According to~\eqref{eq:C}, we evaluate the importance of a vertex based on how much it can reflect its own neighborhood and how much it can discriminate from an arbitrary neighborhood. In the figure, the grey area implements the affinity network $T_w(\cdot,\cdot)$ for both a positive pair (first term in~\eqref{eq:C}) and a negative pair (second term in~\eqref{eq:C}). Given the evaluation of the criterion function $C(\M)$, we use the greedy algorithm to solve the vertex-selection optimization~\eqref{eq:problem} and obtain informative vertices.
}
\label{fig:pooling}
\end{figure*}

\subsection{Mutual information neural estimation}
We first define a vertex's neighborhood. For a vertex $v \in \mathcal{V}$, its neighborhood is a subgraph, $G_v = (\mathcal{N}_v, \Adj_{\mathcal{N}_v})$. The corresponding vertex set
$\mathcal{N}_v = \{ u \in \mathcal{V}~|~d(u,v) \leq R \}$ contains the vertices whose geodesic distances to $v$ are no greater than a threshold $R$. The corresponding graph adjacency matrix $\Adj_{\mathcal{N}_v}$ is the corresponding submatrix of $\Adj$. We call $G_v$ is vertex $v$'s neighborhood; corresponding, $v$ is the anchor vertex of the neighborhood $G_v$.

We then define vertex's features and neighborhood's features. For vertex $v$, the feature is the given signal coefficients $\s_v \in \R^L$ in~\eqref{eq:graph_signals}. For its neighborhood, $G_v$, its feature includes both the connectivity information and vertex features in $\mathcal{N}_v$, denoted as $\q_{\mathcal{N}_v} = [\Adj_{\mathcal{N}_v}, \{ \s_\nu \}_{\nu \in \mathcal{N}_v}]$. Given an arbitrary vertex set $\M \subset \V$, we have access to the associated vertex features $\{ \s_v, v \in \M \}$ and the associated neighborhood features $\{ \q_{\N_v}, v \in \M \}$ whose anchors are in $\M$.

We now want to quantify the dependency between the vertex features and the neighborhood features in the vertex set $\M$, which could be quantified by the mutual information. Let a random variable $\vv$ be the feature of a randomly picked vertex in $\M$, the distribution of $\vv$ is $P_{\vv} = P(\vv = \s_v)$, where $\s_v$ is the outcome feature when we pick vertex $v$. Similarly, let a random variable $\n$ be the neighborhood feature of a randomly picked 
anchor vertex in $\M$, the distribution of  $\n$ is $P_{\n} = P( \n = \q_{\N_u})$, where $\q_{\N_u}$ is the outcome feature when we pick vertex $u$'s neighborhood. The mutual information between selected vertices and neighborhoods is the KL-divergence between the joint distribution $P_{\mathbf{v},\mathbf{n}} = P( \mathbf{v} = \mathbf{s}_v, \mathbf{n} = \mathbf{q}_{\mathcal{N}_v})$ and the product of marginal distributions $P_{\mathbf{v}} \otimes P_{\mathbf{n}}$; that is,
\begin{eqnarray}
    I^{(\M)} \left( \mathbf{v}, \mathbf{n} \right)
    & = & D_{\rm KL} \left( P_{\mathbf{v}, \mathbf{n}} || P_{\mathbf{v}} \otimes P_{\mathbf{n}} \right)
    \nonumber \\  \nonumber 
    & \stackrel{(a)}{\geq} & 
    \sup_{T\in\mathcal{T}} \Big\{  
    \mathbb{E}_{\mathbf{s}_v, \mathbf{q}_{\mathcal{N}_v}  \sim  P_{\mathbf{v}, \mathbf{n}}} \left[ T( \mathbf{s}_v,  \mathbf{q}_{\mathcal{N}_v} ) \right] 
    - \\
    && \mathbb{E}_{ \mathbf{s}_v \sim  P_{\mathbf{v}}, \mathbf{q}_{\mathcal{N}_u} \sim  P_{\mathbf{n}} } \left[ e^{T \left( \mathbf{s}_v,  \mathbf{q}_{\mathcal{N}_u} \right)-1} \right]
    \Big\},
\end{eqnarray}
where $(a)$ follows from $f$-divergence representation based on KL divergence~\cite{BelghaziBROBCH:18};  $T\in\mathcal{T}$ is an arbitrary function that maps the features of a pair of vertex and neighborhood to a real value, here reflecting the dependency between vertex and neighborhood.

Since our goal is to propose a vertex-selection criterion based on the dependencies between vertices and neighborhoods, it is not necessary to compute the exact mutual information based on KL divergence. Instead, we can use non-KL divergences to achieve favourable flexibility and convenience in optimization, which shares the same $f$-representation framework~\cite{NowozinCT:16}. Here we consider a GAN-like divergence.
\begin{eqnarray*}
\label{eq:GAN}
    I^{(\M)}_{\rm GAN} \left( \mathbf{v}, \mathbf{n} \right)
    & = &
    \sup_{T\in\mathcal{T}} \bigg\{  
    \mathbb{E}_{P_{\mathbf{v}, \mathbf{n}}} 
    \big[ \log \sigma \left( T( \mathbf{s}_v, \mathbf{q}_{\mathcal{N}_v} ) \right) \big]  +
   \\
   &&   \mathbb{E}_{P_{\mathbf{v}}, P_{\mathbf{n}} } \big[ \log \left( 1 - \sigma \left( T \left( \mathbf{s}_v, \mathbf{q}_{\mathcal{N}_u}\right) \right) \right) \big]
    \bigg\},
\end{eqnarray*}
where $\sigma(\cdot)$ is the sigmoid function. Intuitively, the function $T(\cdot,\cdot)$ evaluates the affinity between a vertex and a neighborhood. In practice, we cannot go through the entire function space $\mathcal{T}$; instead, we can parameterize $T(\cdot,\cdot)$ by a neural network $T_w(\cdot,\cdot)$, where $w$ denotes the trainable parameters. Through optimizing over $w$, we obtain a neural estimation of $I^{(\M)}_{\rm GAN}$, denoted as $\widehat{I}^{(\M)}_{\rm GAN}$.

We define our vertex selection criterion function to be this neural estimation; that is,
\begin{eqnarray}
\label{eq:C}
C \left( \M \right)  & = & 
\widehat{I}^{(\M)}_{\rm GAN}  \ = \
\max_{w}  \frac{1}{|\M|} 
\sum_{v \in \M}  
\log \sigma \left( T_w(   {\bf s}_v, {\bf q}_{\mathcal{N}_v}) ) \right)  +
\nonumber \\ 
&& 
\frac{1}{|\M|^2} 
\sum_{(v,u) \in  \M }  
\log  \big( 1 - \sigma \left(T_w({\bf s}_v, {\bf q}_{\mathcal{N}_u} )) \right)  \big).
\end{eqnarray}
In $C(\M)$, the first term reflects the average affinities between vertices and their own neighborhoods in  the vertex set $\M$; and the second term reflects the discrepancies between vertices and arbitrary neighborhoods. Notably, a higher $C$ score indicates that a vertex maximally reflects its own neighborhood and meanwhile minimally reflects all the other neighborhoods. 

To specify the affinity network $T_w(\cdot,\cdot)$, we consider 
\begin{equation*}
T_w \left(  \mathbf{s}_v, \mathbf{q}_{\mathcal{N}_u} \right)  = \mathcal{S}_w  \left( \mathcal{E}_w(\mathbf{s}_v), \mathcal{P}_w(\mathbf{q}_{\mathcal{N}_u} ) \right),
\end{equation*}
where the subscript $w$ indicates the associated functions are trainable\footnote{The trainable parameters in $\mathcal{E}_w(\cdot)$,  $\mathcal{P}_w(\cdot)$, and $\mathcal{S}_w(\cdot, \cdot)$ do not share weights.}, $\mathcal{E}_w(\cdot)$ and $\mathcal{P}_w(\cdot)$ are embedding functions for a vertex and a neighborhood, respectively, and $\mathcal{S}_w(\cdot, \cdot)$ is an affinity function to quantify the affinity between a vertex and a neighborhood; see an illustration in  Figure~\ref{fig:pooling}. We implement $\mathcal{E}_w(\cdot)$ and $\mathcal{S}_w(\cdot, \cdot)$ by multi-layer perceptrons (MLPs)~\cite{Goodfellow:2016}, which is a standard fully-connected neural network structure. We further implement $\mathcal{P}_w(\cdot)$ by aggregating vertex features and neighborhood connectivities in ${\bf q}_{\mathcal{N}_u}$; that is
\begin{equation*}
\mathcal{P}_w({\bf q}_{\mathcal{N}_u}) ~=~
\frac{1}{R}
\sum_{r=0}^{R}
\sum_{\nu\in\mathcal{N}_u}
\left( \Adj^r
\right)_{\nu,u}
{\W}^{(r)}
\mathcal{E}_w({\bf s}_{\nu}),~\forall u \in \M,
\end{equation*}
where  $\W^{(r)}\in\mathbb{R}^{d \times d}$ is the trainable weight matrix associated with the $r$th hop of neighbors. Here $\mathcal{P}_w(\cdot)$ is a variant of graph convolutional network~\cite{KipfW:17}.

In our method, the form of mutual information estimation and maximization is similar to deep graph infomax (DGI)~\cite{VelickovicFHLBH:19}, where both DGI and the proposed method use the techniques of mutual information neural estimation~\cite{BelghaziBROBCH:18,HjelmFLGBTB:19} to the graph domain; however, there are three major differences. First, DGI aims to train a graph embedding function while the proposed method aims to evaluate the importance of a vertex via its affinity to its neighborhood. Second, DGI considers the relationship between a vertex and an entire graph while we consider the relationship between a vertex and a neighborhood. By varying the neighbor-hop $R$ in $\mathcal{N}_u$, 
the proposed method is able to tradeoff local and global information. Third, DGI has to train on multiple graphs while the proposed one can work with an individual graph.

\subsection{Vertex selection optimization}
To evaluate the vertex selection criterion $C(\M)$ in~\eqref{eq:C} and solve the maximization problem, which naturally optimizes the internal affinity network $T_w(\cdot, \cdot)$. Note that i) the proposed network is based on the estimation of the dependencies between vertices and neighborhood, exploiting the graph-related property; and ii) the proposed network aggregates information from two sources: graph structures and graph signals.

We now can select the most informative vertex set according to the criterion function~\eqref{eq:C}. The vertex selection optimization problem is
\begin{equation}
\label{eq:problem}
    \max_{\M \subset \mathcal{V}}~~~ C \left( \M \right),~~~{\rm subject~to~~~} |\M| \ = \ M.
\end{equation}
To solve~\eqref{eq:problem}, we consider the submodularity of mutual information~\cite{ChenHHKK:15} and employ a greedy algorithm. We select the first vertex with maximum $C(\M)$ with $|\M|=1$; and we next add a new vertex sequentially by maximizing $C(\M)$ greedily;  however, it is computationally expensive to evaluate $C(\M)$ for two reasons: (i) for any vertex set $\M$, we need to solve an individual optimization problem; and (ii) the second term of $C(\M)$ includes all the pairwise interactions involved with quadratic computational cost. To address Issue (i), we set the vertex set to all the vertices in the graph, maximize $\widehat{I}^{(\mathcal{V})}_{\rm GAN}$ to train the affinity network $T_w(\cdot, \cdot)$. We then fix this network and evaluate $\widehat{I}^{(\M)}_{\rm GAN}$. To address Issue (ii), we perform negative sampling, approximating the second term~\cite{MikolovSCCD:13}.

Note that the proposed neural network $T_w(\cdot,\cdot)$ only provides the vertex-selection criterion, which serves as the objective function in the optimization problem~\eqref{eq:problem}. The subsequent selection algorithm is not part of the network.

\subsection{Training details}
After solving the vertex selection problem~\eqref{eq:problem}, we obtain $\M$ that contains $M$ unique vertices selected from $\mathcal{V}$, leading to the sampling operator $\Psi$ in~\eqref{eq:neural_sampling}. In this way, we train the graph neural sampling module without considering the subsequent recovery task. Since the supervision is simply the relationships between vertices and neighborhoods in the graph structure and we do not use any additional labels as the network supervision, we consider this an unsupervised training setting.

We could also make the affinity network $T_w(\cdot, \cdot)$ be aware of the final recovery error, we introduce a trainable attention vector $\a \in [0,1]^N$ to provides a path to flow gradients for training $\psi_w (\cdot,\cdot)$.  The attention vector shares the same network with $\Psi$ and acts as a bridge connecting a graph neural sampling module and a subsequent graph neural recovery module. For vertex $v \in \V$, the attention score is
\begin{equation}
   \label{eq:weight_score}
    \a_v ~=~ 
    \sigma
    \left(T_w(\s_v, \q_{\N_v})
    \right),
\end{equation}
which measures the affinity between a vertex and its own neighborhood. With the attention vector ${\bf a}$, we assign an trainable importance to each vertex, provide a path for back-propagation to flow gradients, and  unify the training of a graph neural sampling module and a graph neural recovery module. Therefore, we can train the affinity network $T_w(\cdot, \cdot)$ by using the supervisions of both the vertex selection criterion function and the final recovery error. Since we include the recovery loss, we consider it as a supervised training setting.
 
When we sample a graph signal $\x \in \R^N$, we first sample the signal coefficients supported on the vertex set $\M$ and then use the attention vector $\a$ to adjust the measurements. The final samples is
\begin{equation}
\label{eq:graph_neural_sampling_measurements}
\y = \Psi {\rm diag}(\a) \x = \Psi \left(\a \odot \x \right) \in \R^M,
\end{equation}
where the sampling operator $\Psi$ is associted with the selected vertices $\M$ and ${\rm diag}(\a) \in \R^{N \times N} $ is a diagonal matrix. Note that i) $\Psi {\rm diag}(\a) \in \R^{M \times N}$ is a weighted sampling operator, where $\Psi$ provides the selected vertex indices and $\a$ provides the weights; and ii) the selected vertex indices in $\Psi$ are not a direct output from a neural network. They are obtained by solving the optimization~\eqref{eq:problem}, whose objective function is provided by a  neural network. In~\eqref{eq:neural_sampling}, we use a graph neural sampling module $\psi_w(\cdot, \cdot)$ to represent the entire process for simplicity. 

\subsection{Relations to analytical sampling}
Analytical sampling usually provides an optimal or near-optimal solution for sampling a graph signal generated from some specific graph signal model. Some widely-used models include the bandlimited class~\cite{ChenVSK:15}, the approximatley-bandlimited class~\cite{ChenVSK:15c}, the piecewise-smooth class~\cite{ChenVSK:16} and the periodic class~\cite{TanakaE:20}. Many works also extend the sampling setting from subsampling  to local-neighborhood sampling~\cite{WangCG:16}, aggregation sampling~\cite{MarquesSGR:15} and generalized sampling~\cite{TanakaE:20}. However, a fundamental issue of those previous works is that the ground-truth graph signal model in a specific task might be far away from our assumption, causing a significant performance gap between theory and practice. 

To cope with this issue, the proposed graph neural sampling module is data-adaptive and leverages the training ability of neural networks to implicitly capture the underlying graph signal model. At the same time, our method is different from some recent graph pooling methods proposed in the deep learning literature~\cite{YingYMRHL:18,LeeLK:19,DengZWZF:20} from three aspects. First, the previous graph pooling methods only depend on graph signals, while our method considers both graph structure and graph signals. Second, the previous graph pooling methods purely rely on trainable neural networks to directly generate sampling operators without exploring any graph-related properties, while our method explicitly considers the dependency between vertices and local neighborhoods in a graph. A trainable neural network is only used to estimate the mutual information and the output from this network has an interpretable semantic meaning. Third, the previous graph pooling methods have to have a subsequent task to provide additional supervision, while our method can be either supervised and unsupervised as the network is optimized via mutual information neural estimation.

\section{Graph Neural Recovery Module}
\label{sec:recovery}
In this section, we are going to design a neural network $\phi_w (\cdot,\cdot)$ to
recover a complete graph signal from a few measurements collected by the graph neural sampling module $\psi_w (\cdot,\cdot)$. Two key challenges include: i) we need to design an appropriate architecture to recover unknown signal coefficients; and ii) we need to enable end-to-end learning, so that both $\psi_w (\cdot,\cdot)$ and $\phi_w (\cdot,\cdot)$ can be trained together. Here we design the graph neural recovery module $\phi_w (\cdot,\cdot)$ based on the algorithm-unrolling technique, which transform an analytical graph recovery algorithm to a neural network. In the end, we use the recovery error as the supervision to train $\psi_w (\cdot,\cdot)$ and $\phi_w (\cdot,\cdot)$ simultaneously.

\subsection{General analytical graph signal recovery}
We start with an analytical graph signal recovery algorithm. Let $h(\Adj) = \sum_{\ell} h_{\ell} \Adj^{\ell} \in \R^{N \times N}$ be a graph filter with $h_{\ell}$ filter coefficients, which is a polynomial of the graph shift operator. We aim to solve the following optimization problem for recovery:
\begin{subequations}
\label{eq:graph_signal_recovery}
\begin{eqnarray}
  \label{eq:recovery_obj}
	\widehat{\x} &=& \arg \min_{\x} \,  \left\| h(\Adj) \x \right\|_2^2, 
	\\ 
  \label{eq:recovery_cond}  
	\text{subject to}&& \Psi \left( \x \odot \a \right) = \y.
\end{eqnarray}
\end{subequations}
The objective function~\eqref{eq:recovery_obj} promotes that the signal response after graph filtering should be small and
the constraint~\eqref{eq:recovery_cond} requires that when we use the same procedure to sample a recovered graph signal, the measurements should be the same with the current measurements. The design of  filter coefficients depends on a specific graph signal model. For example, to minimize the graph total variation, we can choose a  high-pass graph Haar filter; that is, $h(\Adj) = \Id - \Adj$~\cite{ChenTFVK:18}. Here we use a general form for generality.

We can obtain the closed-form solution for the optimization~\eqref{eq:graph_signal_recovery}. Without loss of generality, we assume that the known coefficients $x_1,\ldots,x_M$ correspond to the first $M$ graph nodes $v_1,\ldots,v_M;$ this arrangement can always be achieved by reordering nodes. Then, a graph signal is
\begin{equation*}
\x = \begin{bmatrix}
\x_\M \\
\x_\U
\end{bmatrix},
\end{equation*}
where $\x_\M \in \R^M$ is the measured part of the signal and $\x_\U \in \R^{N-M}$ is the unknown part to be recovered. Let $\HH = h(\Adj)^T h(\Adj) \in \R^{N \times N}$.  Correspondingly, we represent  $\HH$ in a block form as
\begin{equation*}
\HH = \begin{bmatrix}
\HH_{\M\M} & \HH_{\M\U}  \\
\HH_{\U\M} & \HH_{\U\U}
\end{bmatrix}.
\end{equation*}

\begin{myLem}
The closed-form solution for the graph signal recovery problem~\eqref{eq:graph_signal_recovery} is given by
\begin{equation}
\label{eq:recovery_cf_solution}
\widehat{\x}_\M = \y / \a_\M,~~~
\widehat{\x}_\U =  - \HH_{\U\U}^{-1}\HH_{\U\M} \y_\M / \a_\M,
\end{equation}
where $/$ is the element-wise division.
\end{myLem}
\begin{proof}
we can rewrite the constraint in~\eqref{eq:recovery_cond} as $\x_\M \odot \a_\M = \y$. Then, all the feasible solutions should satisfy $\x_\M  = \y / \a_\M$. We next represent the objective function in~\eqref{eq:recovery_obj} as
\begin{eqnarray*}
\nonumber
\left\| h(\Adj) \x \right\|_2^2	& = &
\x^T \HH \x \\
\nonumber
& = & \begin{bmatrix} \x_\M^T & \x_\U^T \end{bmatrix}
\begin{bmatrix}
\HH_{\M\M} & \HH_{\M\U}  \\
\HH_{\U\M} & \HH_{\U\U}
\end{bmatrix}
\begin{bmatrix} 
\x_\M^T &  \x_\U^T \end{bmatrix} \\
\label{eq:sAs}
& = &
\x_\M^T \HH_{\M\M} \x_\M
+
\x_\U^T \HH_{\U\M}
\x_\M +
\\ \nonumber
&& \,
\x_\M^T \HH_{\M\U}
\x_\U
+
\x_\U^T \HH_{\U\U}
\x_\U.
\end{eqnarray*}
The minimum of the objective function is found by setting its derivative to zero, which yields the closed-form solution
$\widehat{\x}_\U
= -\HH_{\U\U}^{-1} \HH_{\U\M} \widehat{\x}_\M.$
\end{proof}

To avoid calculating the computationally expensive matrix inversion, we can use an iterative algorithm to obtain the same closed-form solution.
\begin{myAlg}
\label{alg:recovery_iterative_solution}
Set the intial conditions as $k = 0$, $\x_\M^{<0>} = \y / \a_\M$, $\x_\U^{<0>} = 0$, and do the following three steps iteratively
\begin{eqnarray}
\label{eq:GTVM_Update}
\x^{<k+1>} & \ \leftarrow \ & \left( \Id - \alpha \HH \right) \x^{<k>},
 \\
\x_\M^{<k>} & \ \leftarrow \ & \y / \a_\M,
\nonumber \\
 k & \ \leftarrow \ & k + 1,
\nonumber
\end{eqnarray}
until convergence, where $\alpha$ is an appropriate step size.
\end{myAlg}

The following theorem analyzes the convergence property of Algorithm~\ref{alg:recovery_iterative_solution}.
\begin{myThm}
\label{thm:equivalency}
Let the iteration number $k \to \infty$. The solution of Algorithm~\ref{alg:recovery_iterative_solution} converges to~\eqref{eq:recovery_cf_solution} with the convergence error decreasing as $O( ||\widetilde{\HH}_{\U\U}||_2^k)$, where $\widetilde{\HH} = \Id - \alpha \HH$.
\end{myThm}

\begin{proof}
We can rewrite the updating step~\eqref{eq:GTVM_Update} as 
\begin{eqnarray}
\x_\U^{<k+1>} & \ = \ & \widetilde{\HH}_{\U\M}\x_\M^{<k>} + \widetilde{\HH}_{\U\U}\x_\U^{<k>}
\nonumber \\ \nonumber
& \ \stackrel{(a)}{=} \ & \sum_{i=0}^{k} {\widetilde{\HH}_{\U\U}^i} \widetilde{\HH}_{\U\M} \x_\M^{<0>}
 + \widetilde{\HH}_{\U\U}^{k+1} \x_\U^{<0>},
\end{eqnarray}
where $(a)$ comes from the induction. Because of initialization $\x_\U^{<0>} =0$, the second term is zero. We then obtain
\begin{eqnarray}
 && \lim_{k \to +\infty} \x_\U^{<k+1>}  = \lim_{k \to +\infty}  \sum_{i=0}^{k} {\widetilde{\HH}_{\U\U}^i} \widetilde{\HH}_{\U\M} \x_\M^{<0>}
\nonumber \\
& \stackrel{(a)}{=} &  (\Id_{\U\U} - \widetilde{\HH}_{\U\U})^{-1} \widetilde{\HH}_{\U\M} \x_\M^{<0>}
\ = \ -\HH_{\U\U}^{-1} \HH_{\U\M} \x_\M^{<0>},
\nonumber
\end{eqnarray}
where $(a)$ comes from the property of the matrix inversion. Denote $\x_\U \ = \ \lim_{k \to +\infty} \x_\U^{<k>}$, we have
\begin{eqnarray*}
|| \x_\U^{<k>} - \x_\U ||_2 & \ = \ & ||\sum_{i=k}^{+\infty} {\widetilde{\HH}_{\U\U}^i} \widetilde{\HH}_{\U\M} \x_\M^{<0>} ||_2
\nonumber \\
& \ = \ & || \widetilde{\HH}_{\U\U}^k \sum_{i=0}^{+\infty} {\widetilde{\HH}_{\U\U}^i} \widetilde{\HH}_{\U\M} \x_\M^{<0>} ||_2
\nonumber \\
& \ = \ & || \widetilde{\HH}_{\U\U}^k \x_\U ||_2
 \ \stackrel{(a)}{\leq} \  || \widetilde{\HH}_{\U\U}||_2^k || \x_\U ||_2,
\end{eqnarray*}
where $(a)$ follows the definition of spectral norm.
We finally obtain
\begin{eqnarray}
\nonumber
\frac{|| \x_\U^{<k>} - \x_\U ||_2}{|| \x_\U ||_2} & \ \leq \ & || \widetilde{\HH}_{\U\U}||_2^k.
\end{eqnarray}
\end{proof}

The following theorem shows that it is fairly easy to choose $\alpha$ to lead to the convergence.
\begin{myThm}
\label{thm:convergence}
Let ~$0 < \alpha \leq 2/\lambda_{max}(\widetilde{\HH})$. Then,
$
||\widetilde{\HH}_{\U\U}||_2 \leq 1.
$
\end{myThm}
\begin{proof}
Since $\HH$ is real and symmetric, we decompose it as
\begin{eqnarray}
\label{eq:A_decomp}
 \HH & = & \Um \Lambda \Um^T,
\end{eqnarray}
where $\Um$ is an unitary matrix and $\Lambda$ are a diagonal matrix whose diagonal elements are are the corresponding eigenvalues, denoted as $\lambda_i$; moreover, $\lambda_1 \geq \lambda_2 \geq \ldots \geq \lambda_N \geq 0$. Since $0 < \alpha \leq 2/\lambda_{max}(\widetilde{\HH})$,  we have 
\begin{eqnarray}
\label{eq: alpha_ineq}
 1 \geq 1- \alpha \lambda_i \geq 1 - 2\lambda_i/\lambda_{max}(\HH) \geq -1.
\end{eqnarray}
We then bound the spectral norm of $\widetilde{\HH}_{\U\U}$ as
\begin{eqnarray}
	||\widetilde{\HH}_{\U\U}||_2  & \stackrel{(a)}{\leq} &||\widetilde{\HH}||_2 
	\ \stackrel{(b)}{=} \ ||   \Id - \alpha \HH  ||_2 
	\nonumber \\
	& \stackrel{(c)}{=} & ||    \Id - \alpha  \Um \Lambda \Um^T ||_2 
	\nonumber \\
	& \stackrel{(d)}{=} & ||   \Um  ( \Id - \alpha\Lambda) \Um^T ||_2 		
	\nonumber \\ \nonumber
	& \stackrel{(e)}{=} & \max_i  | 1- \alpha \lambda_i |  \stackrel{(f)}{\leq}  1,
\end{eqnarray}
where $(a)$ comes from the property of spectral norm,  $(b)$ follows the definition of $\widetilde{\HH}$, $(c)$ follows~\eqref{eq:A_decomp}, $(d)$ comes from the unitary property of $\Um$, $(e)$  comes from the definition of spectral norm and $(f)$ follows~\eqref{eq: alpha_ineq}.
\end{proof}
We now obtain an iterative algorithm to solve the recovery optimization~\eqref{eq:graph_signal_recovery}. The key step~\eqref{eq:GTVM_Update} is to use $\widetilde{\HH}$ to filter the solution $\x^{<k>}$ from the previous step and $\widetilde{\HH} = \Id - \alpha h(\Adj)^T h(\Adj)$ is essentially a polynomial of $\Adj$, which is again a graph filter. Then a fundamental challenge is now to choose the filter coefficients to adapt to a specific task. Previously we have to manually design the filter coefficients for a predefined graph signal model. Now with the neural network framework, we are going to adaptively learn those filter coefficients from data.

\subsection{Algorithm unrolling}
Algorithm unrolling provides a concrete and systematic connection between iterative algorithms in signal processing and deep neural networks, paving the way to developing interpretable network architectures. The core strategy is to unroll an iterative algorithm into a graph neural network by mapping each iteration into a single network layer and stacking multiple layers together. 

To unroll Algorithm~\ref{alg:recovery_iterative_solution}, we substitute the fixed $\widetilde{\HH}$ by a trainable graph filter.  We then consider the $k$-th network layer in the  neural recovery module works as
\begin{eqnarray}
\label{eq:unrolling_recovery}
\x^{<k+1>} & \ \leftarrow \ & \sum_{\ell=1}^L h_{\ell}^{<k>}
\Adj^{\ell}\x^{<k>},
\\
\x_\M^{<k>} & \ \leftarrow \ & \y / \a_\M,
\nonumber
\end{eqnarray}
where filter coefficients $
h_{1}^{<k>}, h_{2}^{<k>}, \cdots, h_{L}^{<k>} \in \R$ at each layer are trainable parameters. We can also consider a multi-channel setting and some other variants of a graph filter~\cite{GamaMLR:19}. 

After forward-propagation through $K$ layers, we output $\widehat{\x} = \x^{<K+1>}$ as the final solution of the neural recovery module.

\subsection{Training details}
We then can train the proposed graph neural sampling and recovery modules together. Given a set of $L$ graph signals~\eqref{eq:graph_signals}, 
we consider the following training loss
\begin{eqnarray}
\label{eq:recovery_loss}
\mathcal{L} ~=~ \mathcal{L}_{\rm recovery} + \alpha \mathcal{L}_{\rm sample},
\end{eqnarray}
where a recovery-error loss
\begin{eqnarray*}
\mathcal{L}_{\rm recovery} 
\ = \ 
\sum_{i=1}^L \left\| \x^{(i)} - {\phi}_w \Big( \Adj, \Psi \left( \x^{(i)} \odot \a \right) \Big) \right\|_2^2,
\end{eqnarray*}
with $\Psi$ and $\a$ following from the graph neural sampling module~\eqref{eq:graph_neural_sampling_measurements}, and ${\phi}_w$ following from the graph neural recovery module~\eqref{eq:unrolling_recovery}, and a vertex-selection loss based on mutual information neural estimation
\begin{equation}
\label{eq:mine_loss}
\mathcal{L}_{\rm sample}= - \widehat{I}^{(\mathcal{V})}_{\rm GAN},
\end{equation}
which is the vertex-selection criterion and relates to $\Psi$ and $\a$. The hyper-parameter $\alpha$ balances the overall recovery task and vertex-selection task. Note that our neural sampling module is trained through both $\mathcal{L}_{\rm recovery}$ and $\mathcal{L}_{\rm sample}$, which makes the graph neural sampling module adapt to the final recovery task; we consider this as the supervised setting. When we only  use $\mathcal{L}_{\rm sample}$ to train the graph neural sampling module, it is the unsupervised setting. In the entire process, the trainable components include the affinity network $T_w(\cdot, \cdot)$ for sampling and the graph filter coefficients for recovery.

\subsection{Relations to analytical recovery}
Analytical recovery usually solves
an optimization problem with a predefined graph-regularization term. For example,~\eqref{eq:recovery_obj} regularizes the filter responses. Some other options include the quadratic form of graph Laplacian~\cite{ShumanNFOV:13}, the $\ell_1$ form of graph incident matrix~\cite{WangSST:16}. However, a fundamental issue is that which graph-regularization term is appropriate in a specific recovery task.
An arbitrary graph-regularization term would provide misleading induced bias, causing a significant performance gap between theory and reality.

The proposed graph neural recovery method provides both flexibility and interpretability. On the one hand, it is  data-adaptive and supposed to learn a variety of graph signal models from given graph signals by leveraging deep neural networks. On the other hand, we unroll an analytical recovery algorithm to a graph neural network whose operations are interpretable by following analytical iterative steps.

\begin{figure*}[ht]
\centering
\includegraphics[width=0.9\textwidth]{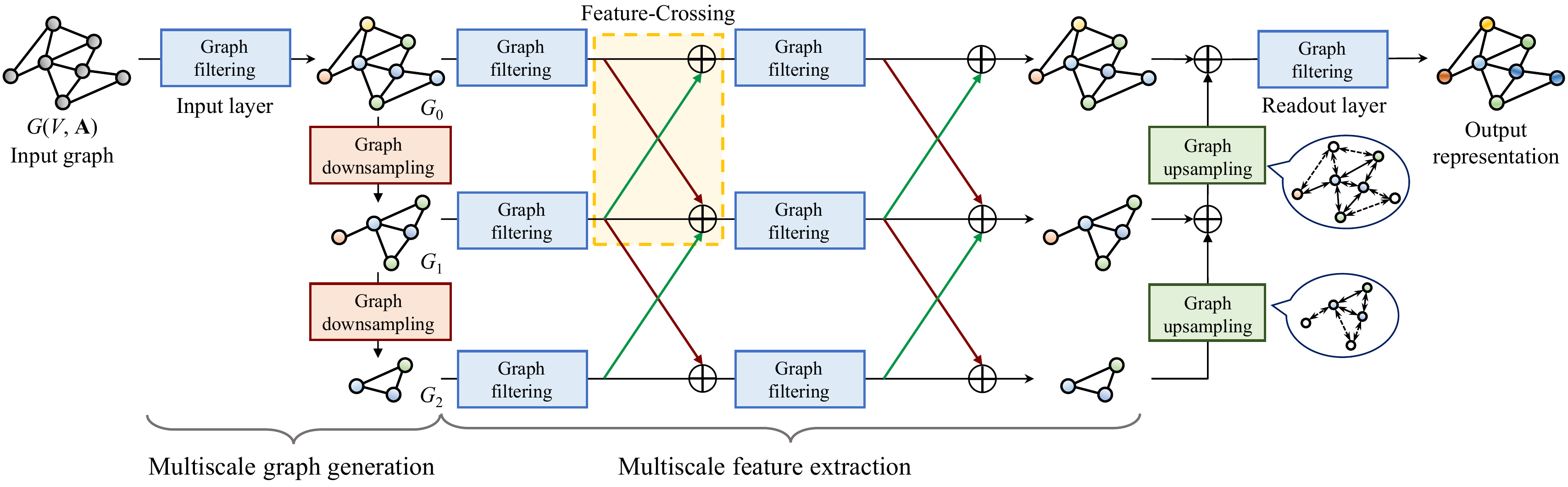}
\caption{\small Implementation of graph cross network (GXN). We show an exemplar model with 3 scales and 4 feature-crossing layers. Graph downsapling and upsampling are designed based on the proposed graph neural sampling and recovery modules, respectively. The proposed feature-crossing layer allows intermediate features at various scales to communicate and merge, promoting information flow across scales.}
\label{fig:MGNN}
\end{figure*}

\section{Multiscale Model: Graph Cross Network}
\label{sec:MGNN}
For analytical graph signal sampling and recovery, one of the biggest applications is to empower a graph filter bank to achieve multiscale analysis. For the proposed graph neural sampling and recovery modules, we can use them to design a multiscale graph neural network.

One benefit of multiscale representations is to achieve flexible tradeoff between local and global information~\cite{HuCOA:15,ShumanFV:16}. A multiscale graph filter bank is an array of band-pass graph filters that separates an input graph signal into multiple components supported at multiple graph scales. As a trainable counterpart of a multiscale graph filter bank, a multiscale graph neural network consists of three basic operations: graph downsampling, graph upsampling and graph filtering, which are all trainable and implemented by neural networks. 
Based on graph downsampling and upsampling, we can achieve multiscale graph generation; and based on graph filtering, we achieve multiscale feature extraction. Moreover, we propose a new feature-crossing layer to promote the fusion of the intermediate features across graph scales, improving the learning ablity. The iconic structure of feature-crossing layer also leads to the model name: graph cross network (GXN). 

\subsection{Multiscale graph generation} 
To build a multiscale graph neural network, we first need to generate multiple scales for a graph and provide the corresponding vertex-to-vertex association across scales. To achieve this, we need to design graph downsampling and upsampling.

\mypar{Graph downsampling} Graph downsampling is a process to compress a larger graph structure with the associated graph signal to
a smaller graph structure with the associated graph signal. It involves  two related, yet different parts: graph structure downsampling and graph signal downsampling.  Here graph signal downsampling is the same with graph signal sampling. During graph downsampling, we have to reduce the number of vertices and edges, causing information loss. A good graph downsampling method should preserve as much information as possible.

The proposed graph neural sampling module naturally selects informative vertices and collects the associated signal coefficients. This learnt sampling operator also provides the direct vertex-to-vertex association across scales. To further achieving graph structure downsampling, we need to connect the selected vertices according to the original connections. Here we consider three approaches:

$\bullet$ Direct reduction, that is, $\Psi \Adj \Psi \in \R^{M \times M}$ where $\Psi$ is the sampling operator. This is simple and straightforward, but loses significant connectivity information;

$\bullet$ Fused reduction, that is, $\Ss \Adj \Ss^T$ with $\Ss={\rm softmax}( \Psi \Adj )\in [0,1]^{M \times N}$, where softmax$(\cdot)$ is a row-wise softmax function, achieving the row-wise normalization Each row of $\Ss$ represents the neighborhood of a selected vertex and the intuition is to fuse the neighboring information to the selected vertices; and

$\bullet$ Kron reduction~\cite{DorflerB:13}, which is the Schur complement of the graph Laplacian matrix and preserves the graph spectral properties, but it is computationally expensive  due to the matrix inversion. We can convert the original graph adjacency matrix to a graph Laplacian matrix, execute Kron reduction to obtain a downsampled graph Laplacian matrix, and convert it back to a downsampled graph adjacency matrix.

In our experiments, we see that Kron reduction leads the effectiveness; the direct reduction leads the efficiency; and the fused reduction achieves the best tradeoff between effectiveness and efficiency. We thus consider the fused reduction as our default.

Overall, we can use graph downsampling several times to generate multiscale representations for a graph. Given an input graph structure $G(\mathcal{V}, \Adj)$ with the associated graph signal, $\x \in \mathbb{R}^N$, we first initialize the finest scale of graph structure as $G^0(\mathcal{V}^0, \Adj^{0})$ with $\mathcal{V}^0 = \mathcal{V}$, $\Adj^{0} = \Adj$ and the associated graph signal $\x^{0} = \x$. We then recursively apply graph downsampling for $S$ times to obtain a series of coarser scales of graph structure $G^1(\mathcal{V}^1, \Adj^1), \dots, G^S(\mathcal{V}^S, \Adj^S)$  and the associated graph signals $\x^{1}, \dots, \x^{S}$ from $G^0$ and $\x^{0}$, respectively, where $|\mathcal{V}^{s}|>|\mathcal{V}^{s'}|$ for $\forall~ 1 \leq s < s' \leq S$. Here the super-script indexes the graph scale.

\mypar{Graph upsampling} Graph upsampling is an inverse process of graph downsampling. Since we have the original graph structure and the vertex-to-vertex association across scales, we only need to design graph signal upsampling, which is equivalent to graph signal recovery. We can directly use the proposed graph neural recovery module to obtain the signal coefficients at the unselected vertices.

After extracting features at each scale, we can use graph upsampling to lift features extracted in a coarser scale to a finer scale. For example, let $G^s(\mathcal{V}^s, \Adj^s)$ and $\h^{s} \in \R^{|\mathcal{V}^s|}$ be the graph structure and the extracted feature vector at the $s$-th scale, respectively. To obtain the corresponding feature vector $\h^{s'} \in \R^{|\mathcal{V}^{s'}|}$ at the $s'$-th scale ($s' < s$), we can recursively apply graph upsampling for $s-s'$ times based on the graph structure at each scale.

The proposed graph downsampling and upsampling together can generate multiscale representations of a graph and enable the feature conversion across scales. Note that both are trainable and adaptive to specific data and tasks.

\subsection{Multiscale feature extraction} 
Given multiple scales of a graph, we build a graph neural network at each scale to extract features. Each  network consists of a sequence of trainable graph filters. After feature extraction at each scale, we combine deep features at all the scales together to obtain the final representation. We use graph upsampling to align features at different scales. We finally leverage a trainble graph filter to synthesize the fused multiscale features and generate the final representation for various downstream tasks, such as graph classification and vertex classification.

To further enhance information flow across scales, we propose a feature-crossing layer between two consecutive scales at various network layers, allowing multiscale features to communicate and merge in the intermediate network layers. Mathematically, let $\h^{<s,k>}$ be the feature vector at the $s$-th scale and the $k$-th network layer. In the same $k$-th network layer, we downsample the feature vector at the $s-1$-th scale to the $s$-th scale and obtain $\h^{<s,k, \downarrow>}$. We also upsample the feature vector at the $s+1$-th scale to the $s$-th scale and obtain $\h^{<s,k, \uparrow>}$. These two steps are implemented by using the proposed graph downsampling  and graph upsampling. After a feature-crossing layer, we add feature vectors from these three sources and obtain the fused feature vector
\begin{equation*}
\h^{<s,k>} \ \leftarrow \
   \begin{cases}
   ~\h^{<s,k>} + \h^{<s,k, \uparrow>}, & s=0,
   \\
   ~\h^{<s,k>} + \h^{<s,k, \downarrow>} + \h^{<s,k, \uparrow>}, & 0<s<S,
   \\
   ~\h^{<s,k>} + \h^{<s,k, \downarrow>}, & s=S.
   \end{cases}
\end{equation*}
Since the proposed feature-crossing layer forms a cross shape, we call this multiscale graph neural network architecture~\emph{graph cross network (GXN)}; see Figure~\ref{fig:MGNN}. Compared to the standard multiscale graph neural network, the intermediate cross connections promote the information flow across multiple scales and improve the performances.

\subsection{Training details}
Here we consider training GXN for two tasks: vertex classification and graph classification. We use the same network architecture to implement multiscale graph generation and feature extraction for both tasks. The only differences between these two tasks are the final output and the loss function. 

In vertex classification, we aim to classify each vertex to one or more predicted categories. The final output of GXN is a predicted labeling vector $\widehat{\z}_v \in \R^C$ for vertex $v$, where $C$ is the number of vertex categories in a graph. We then use the cross-entropy loss between the predicted and ground-truth labeling vector to supervise the training; that is,
\begin{equation}
\label{eq:vertex_classification_loss}
\mathcal{L}_{\rm vertex} \ = \ -\sum_{v} \z_v^{T}\log(\widehat{\z}_v),
\end{equation}
where $\z_v \in \R^C$ is ground-truth labeling vector for vertex $v$.

In graph classification, we aim to classify an entire graph to one or more predicted categories. After obtaining the final feature vector $\widehat{\z}_v$ for each vertex, we use the standard SortPool~\cite{ZhangCNC:18} to remove the vertex dimension and  obtain a graph labeling vector $\widehat{\z}_G \in \R^C$ for a graph $G$, where $C$ is the number of graph categories in a dataset. We then use the cross-entropy loss between the predicted and ground-truth labeling vector to supervise the training; that is,
\begin{equation*}
\mathcal{L}_{\rm graph} \ = \ -\sum_{G} \z_G^{T}\log(\widehat{\z}_G), 
\end{equation*}
where $\z_G \in \R^C$ is ground-truth labeling vector for a graph $G$.

Similarly to the training of sampling and recovery~\eqref{eq:recovery_loss}, we can further add the vertex-selection loss~\eqref{eq:mine_loss} to make the involved graph neural sampling modules adapt to the final task.

\subsection{Relations to graph filter banks}
A graph filter bank uses a series of band-pass graph filters that expands the input graph signal into multiple subband components~\cite{NarangO:12,ShumanWHV:15,ShumanFV:16,SakiyamaWTO:19}. By adjusting the component in each subband, a graph filter bank can flexibly modify and reconstruct a graph signal. The expansion part is called analysis and the reconstruct part is called synthesis. To analyze and synthesize at multiple scales, a multiscale graph filter bank uses analytical graph signal sampling to achieve pyramid representations~\cite{ShumanFV:16}. The benefit is that we can extract useful features at each scale and combine them in the end. A multiscale graph filter bank includes three building blocks: graph downsampling, graph upsampling and graph filtering.

The proposed multiscale graph neural network is essentially a trainable multiscale graph filter bank. The graph downsampling follows from graph neural sampling in Section~\ref{sec:sampling}; the graph upsampling follows from graph neural recovery in Section~\ref{sec:recovery}; and graph filtering follows from trainable graph convolution operation proposed in previous works~\cite{GamaMLR:19,GamaRB:19}. All these three components are trainable and data-adaptive. The analysis and synthesis modules are implicitly implemented during learning. By adjusting the last layer and the supervision for the proposed multiscale graph neural network, we can handle various graph-related tasks. Compared to a conventional multiscale graph filter bank, a multiscale graph neural network is more flexible for a new task.

\section{Illustration of Sampling and Recovery} 
\label{sec:illustration}
In this section, We are going to compare the analytical solutions and the neural-network solutions on a few toy examples. We try to provide some insights on various vertex-selection strategies and the corresponding recovery performances.

\subsection{Vertex selection strategy} 
Here we consider two experiments. First, we consider sampling of bandlimited signals on two similar, yet different graph structures, where we obtain some intuitions about how graph structures influence the vertex selection.  In the second experiment, we consider sampling of bandlimited and piecewise-bandlimited signals on the same graph structure, which helps understand how given graph signals influence the vertex selection.

\mypar{Effect of graph structures}
We generate two graph structures with $2,400$ vertices based on the stochastic block models~\cite{Abbe18}. Each graph has two communities, which have $1,800$ and $600$ vertices, respectively. 
For the first graph, we set the connection probability between vertices in the first community be $2\%$, the connection probability between vertices in the second community be $6\%$, and the connection probability between vertices from two different communities be $0.005\%$. Overall, all the vertices in this graph have similar degrees and the average degree is around $36$. We thus name it the similar-degree graph. For the second graph, we set the connection probability between vertices in the same communities be $2\%$, and the connection probability between vertices from two different communities be $0.005\%$. In each community, the number of edges is approximately proportional to the number of vertices; that is, the average degrees in the first and second communities are around $36$ and $12$, respectively. We thus name it the similar-density graph. 
Given a sampling method, we select $10$ vertices out from $2,400$ vertices. We run $20$ trials and  then compute the corresponding probability that the selected vertices fall into the smaller community in each of two graphs.

We apply three sampling methods: bandlimited-space (BLS) sampling ~\cite{ChenVSK:15}, spectral-proxy (SP) sampling ~\cite{AnisCO:16} and the proposed graph neural sampling module. The selected vertices designed by bandlimited-space sampling aim to maximally preserve information in the bandlimited space, which is spanned by the first $K=10$ eigenvectors of the graph Laplacian matrix. Spectral-proxy sampling promotes the similar idea; however, instead of using the exact eigenvectors, it uses the spectral proxy to approximate the bandlimited space. Spectral-proxy sampling involves a hyperparameter, $k \geq 1$, which is the proxy order. When $k$ is larger, the spectral proxy has a better approximation to the bandlimited space. Both bandlimited-space sampling and spectral-proxy sampling are analytically designed based on explicit graph signal models. For the proposed neural sampling, we use the first $K=10$ eigenvectors of the graph Laplacian matrix as $10$ graph signals to train neural networks; in other words, we guide the neural networks to preserve information in the bandlimited space. Therefore, we expect the proposed graph neural sampling module should have similar performances with bandlimited-space sampling and spectral-proxy sampling.

\begin{table}[!t]
    \centering
    \caption{Probability to select vertices from the smaller community. The proposed graph neural sampling module has similar behaviors with the analytical sampling methods.}
    \setlength{\tabcolsep}{1.36mm}{
    \begin{tabular}{c|c|c}
         \hline 
         & Similar-degree  & Similar-density  \\
         \hline 
        Bandlimited-space sampling & $27.5\%$ & $99.0\%$ \\
        Spectral-proxy sampling, $k=1$   & $29.5\%$ & $27.5\%$  \\
        Spectral-proxy sampling, $k=3$   & $30.0\%$ & $49.0\%$  \\
        Spectral-proxy sampling, $k=5$   & $30/0\%$ & $75.5\%$ \\ 
        Neural sampling with recovery & $44.5\%$ & $70.0\%$ \\
        Neural sampling without recovery & $46.0\%$ & $71.0\%$ \\
         \hline
    \end{tabular}}
    \label{tab:sampling_two_graph_structures}
\end{table}

\begin{figure}[!t]
\centering
\includegraphics[width=0.48\textwidth]{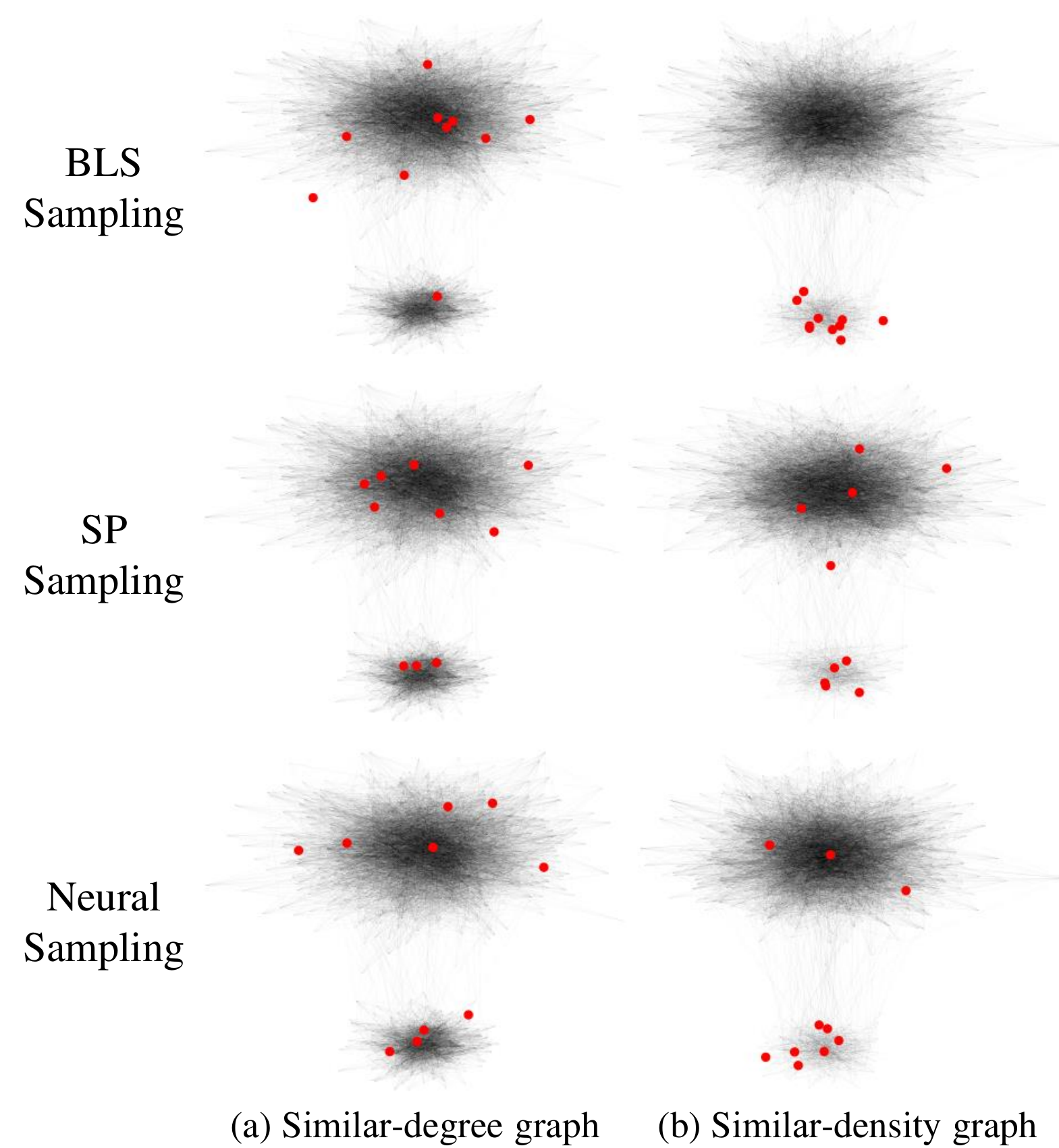}
\caption{\small Vertex selection on a similar-degree graph and a similar-density graph with various sampling methods. The sampled vertices are colored in red. The proposed graph neural sampling has similar behaviors with the analytical sampling methods on both graphs. }
\label{fig:sampling_two_graph_structures}
\end{figure}

Table~\ref{tab:sampling_two_graph_structures} shows the probability that the selected vertices fall into the smaller community in either the similar-degree graph or the similar-density graph. We see that in the similar-degree graph, bandlimited-space sampling selects $27.5\%$ vertices from the smaller community, which has $25\%$ vertices in the entire graph. The intuition is that the vertices in both communities have similar amount of information. In comparison, in the similar-density graph, bandlimited-space sampling selects $99\%$ vertices from the smaller community. The intuition is that the vertices with weaker connectivities are more informative because their information is much harder to be accessed from other vertices. Spectral-proxy sampling shows similar trends, especially when we increase the proxy order $k$. For the
graph neural sampling module, we consider both with and without using the recovery-error loss. Both cases show consistent performances and are similar to bandlimited-space sampling. This reflects that given the same graph signal model, the proposed graph neural sampling module has the similar performances with analytical sampling and it is adaptive to the underlying graph structure. Figure~\ref{fig:sampling_two_graph_structures}
illustrates the vertices selected by sampling methods in both graph structures.

\begin{figure}[!t]
  \vspace{-5pt}
  \begin{center}
    \begin{tabular}{ccc}
     \includegraphics[width=0.45\columnwidth]{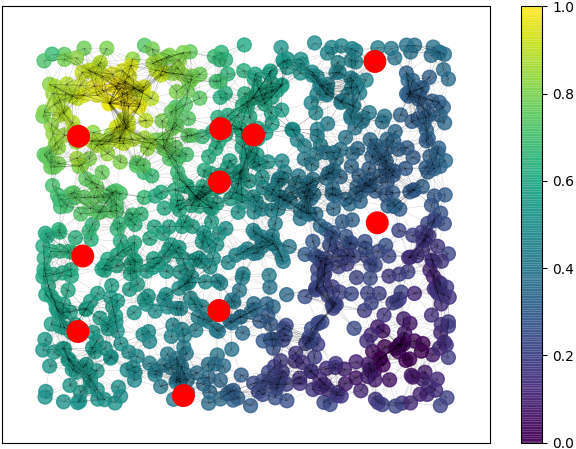} & 
     \includegraphics[width=0.45\columnwidth]{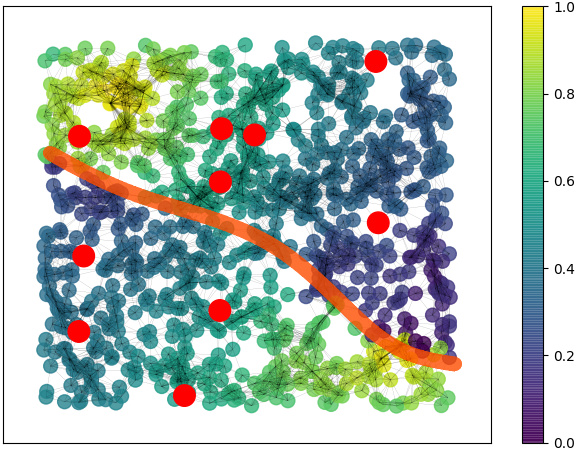} 
     \\
     \tabincell{c}{{\small (a) BLS sampling} \\ {\small on smooth signals.}} & 
     \tabincell{c}{{\small (b) BLS sampling on} \\ {\small piecewise smooth signals.}}
     \\
     \includegraphics[width=0.45\columnwidth]{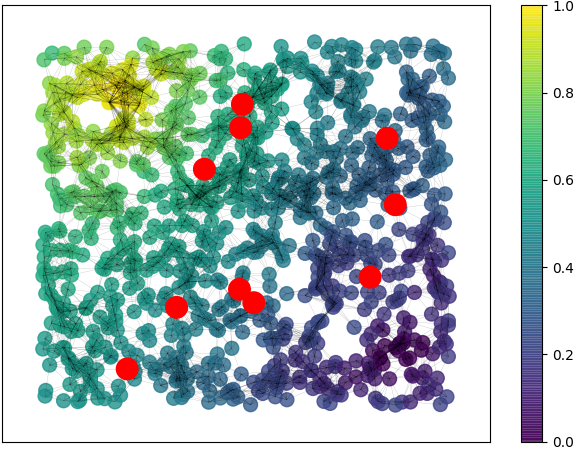} & 
     \includegraphics[width=0.45\columnwidth]{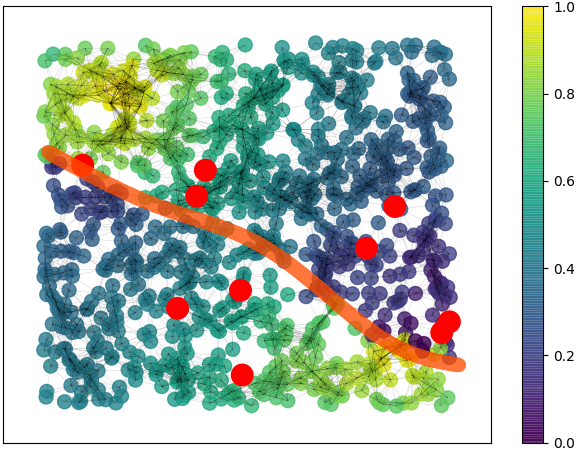}
     \\
     \tabincell{c}{{\small (c) Neural sampling} \\ {\small on smooth signals.}} & 
     \tabincell{c}{{\small (d) Neural sampling on} \\ {\small piecewise smooth signals}.}
  \end{tabular}
\end{center}
\vspace{-10pt}
\caption{\small  Vertex selection provided by bandlimited-space (BLS) sampling and graph neural sampling module. Vertices marked in red are selected. Orange curves in (b) and (d) approximately sketch the boundary between two pieces of smooth graph signals. The proposed graph neural sampling module adaptively selects more vertices around the boundary.}
\vspace{-2pt}
\label{fig:sampling_two_graph_signals}
\end{figure}

\begin{figure}[!t]
  \vspace{-5pt}
  \begin{center}
    \begin{tabular}{cc}
     \includegraphics[width=0.45\columnwidth]{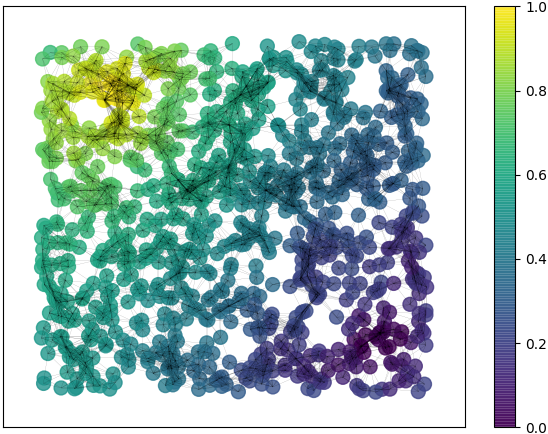} & 
     \includegraphics[width=0.45\columnwidth]{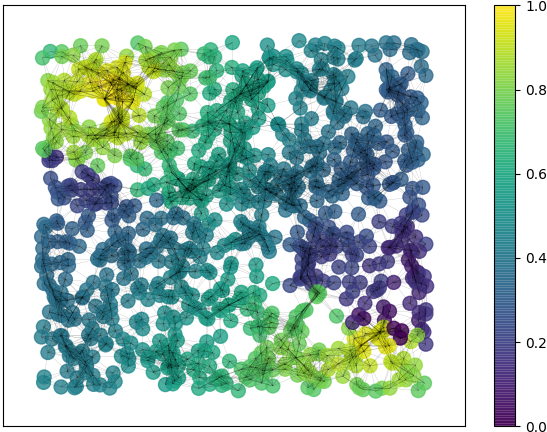}
     \\
     \tabincell{c}{\small (a) Smooth graph signal.}
     & \tabincell{c}{{\small (b) Piecewise-smooth} \\ {\small graph signal}}.
     \\
     \includegraphics[width=0.45\columnwidth]{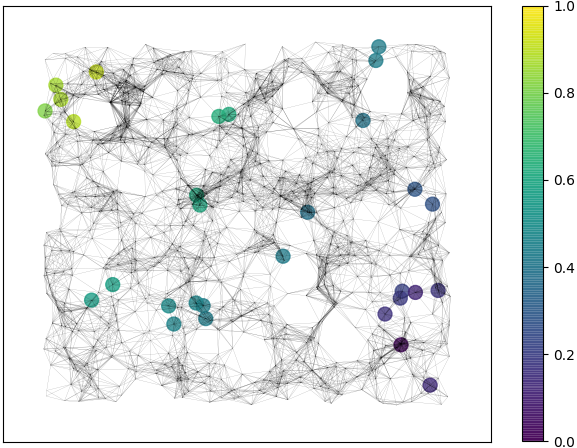}  & 
     \includegraphics[width=0.45\columnwidth]{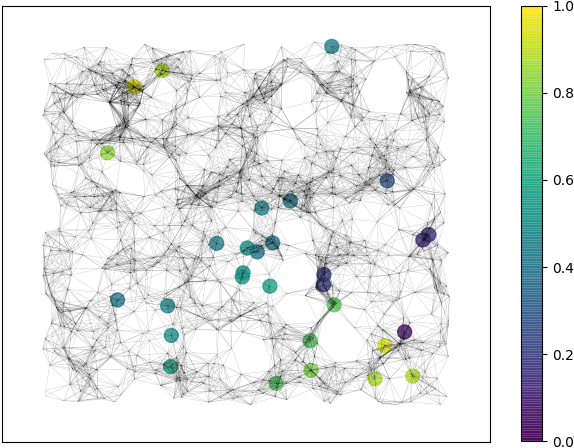}
     \\
     \tabincell{c}{{\small (c) Sampled smooth } \\ {\small graph signal}}  & 
     \tabincell{c}{{\small (d) Sampled piecewise-} \\ {\small smooth graph signal}} 
     \\
     \includegraphics[width=0.45\columnwidth]{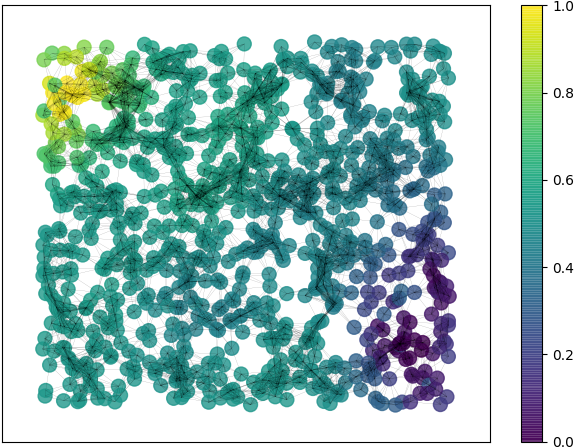}  & 
     \includegraphics[width=0.45\columnwidth]{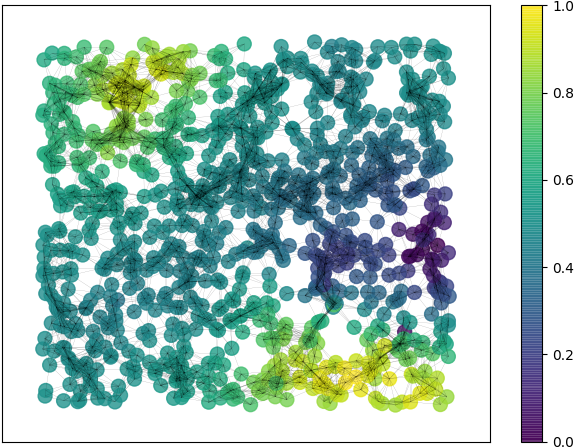}
     \\
     \tabincell{c}{{\small (e) Recovered smooth} \\ {\small graph signal.}} & 
     \tabincell{c}{{\small (f) Recovered piecewise-} \\ {\small smooth graph signal.}}
  \end{tabular}
\end{center}
\vspace{-10pt}
\caption{\small The proposed graph neural sampling and recovery modules well reconstruct both smooth and piecewise-smooth graph signals from $30$ measurements.}
\vspace{-2pt}
\label{fig:recovery_two_graph_signals}
\end{figure}

\mypar{Effect of graph signals}
Based on a geometric graph with $1,000$ vertices, we generate two types of graph signals, bandlimited graph signals, which are the first $10$ eigenvectors of the graph Laplaican matrix, and piecewise-bandlimited graph signals, where we intentionally introduce a boundary by applying a mask to the bandlimited graph signals. We use bandlimited-space sampling and the proposed graph neural sampling module to select $10$ vertices, respectively, and compare the selected vertices in Figure~\ref{fig:sampling_two_graph_signals}. Since bandlimited-space sampling is designed solely based on graph structures, the selected vertices are the same for both bandlimited and piecewise-bandlimited graph signals. On the other hand, the proposed graph neural sampling module  adapts to both graph signals and graph structures. We see that the selected vertices for piecewise-bandlimited graph signals are much closer to the boundary than those vertices for bandlimited graph signals, reflecting the vertices along the boundary are informative.

\subsection{Recovery performance}
We further validate the recovery performance of the proposed graph neural recovery module.  Based on a geometric graph with $1,000$ vertices, we generate two types of graph signals, $10$ bandlimited graph signals and $10$ piecewise-bandlimited graph to train two pairs of the proposed graph neural sampling and recovery modules, respectively. For a new bandlimited or piecewise-bandlimited graph signal, we use the proposed graph neural sampling module to select $30$ vertices and take the corresponding measurements; see Figure~\ref{fig:recovery_two_graph_signals} (c) and (d), respectively. We then apply the proposed graph recovery sampling module to reconstruct the original graph signals. Figure~\ref{fig:recovery_two_graph_signals} (e) and (f) illustrates that the reconstructions well approximates the original graph signals. This validates that i) the proposed graph neural sampling and recovery modules can adapt to various types of graph signals; and ii) the proposed graph neural sampling and recovery modules can be well generalized to new data with small amount of training data.

\begin{figure*}[!t]
  \vspace{-5pt}
  \begin{center}
    \begin{tabular}{ccc}
     \includegraphics[width=0.6\columnwidth]{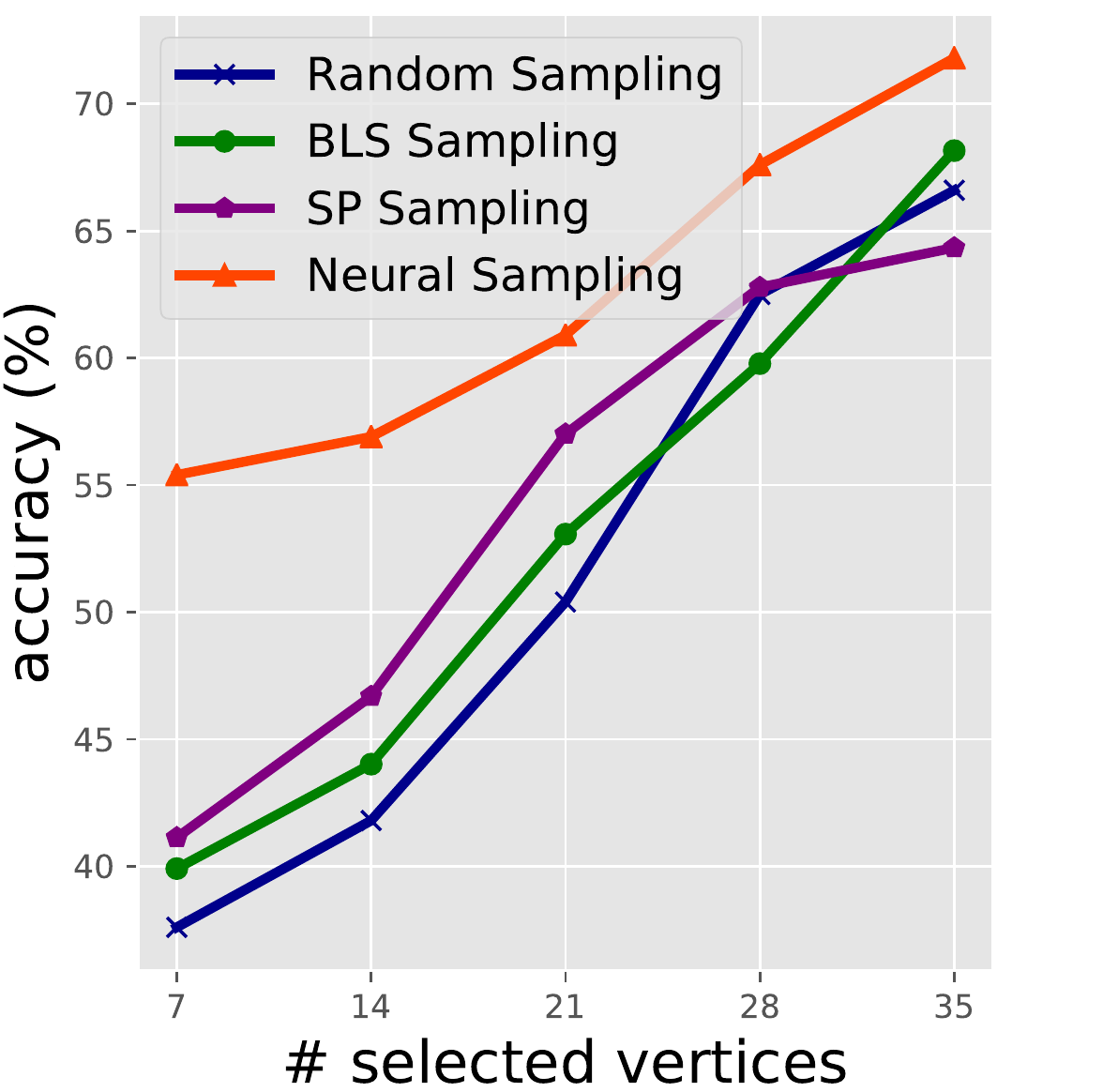} & 
     \includegraphics[width=0.6\columnwidth]{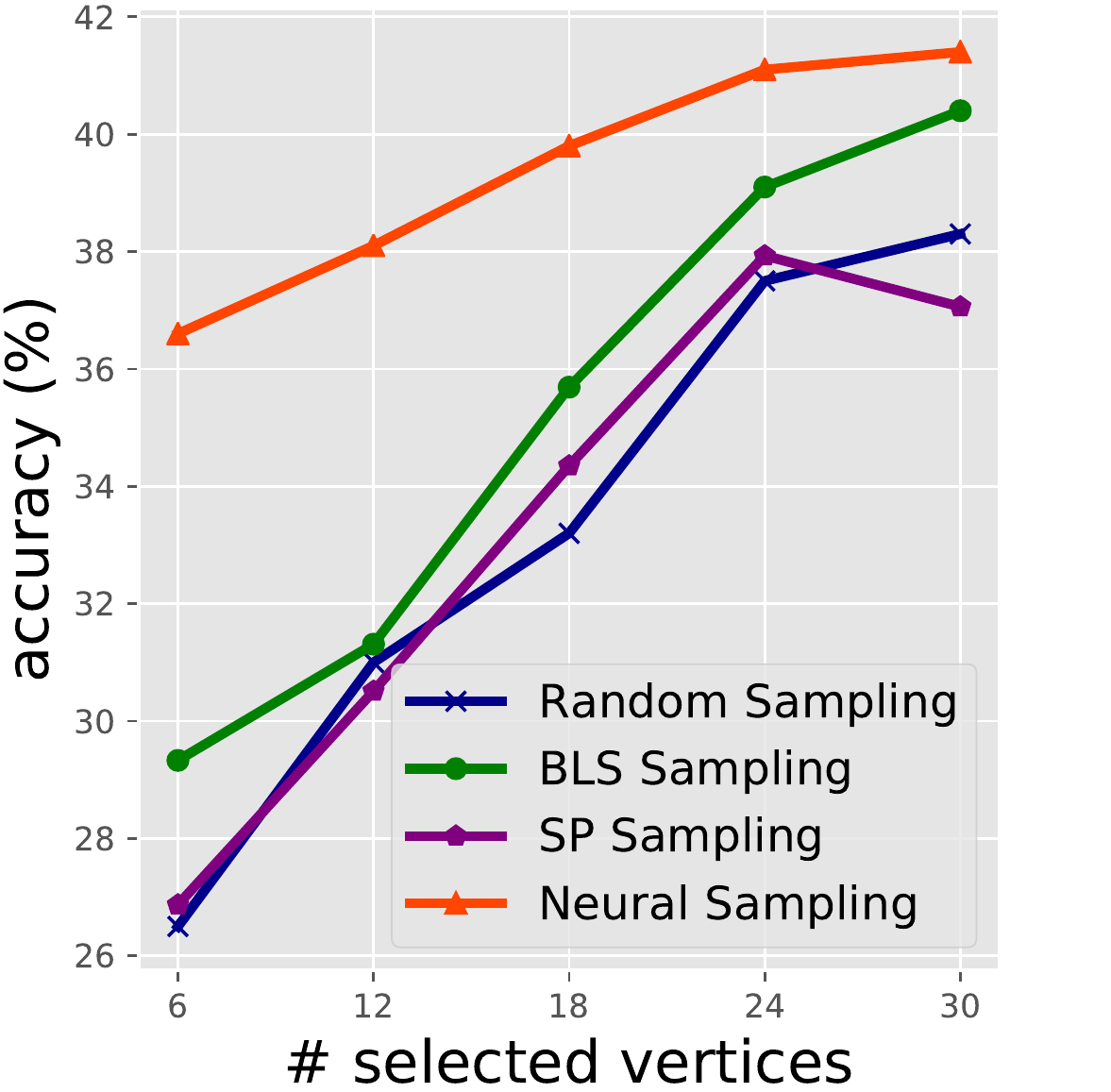} & 
     \includegraphics[width=0.6\columnwidth]{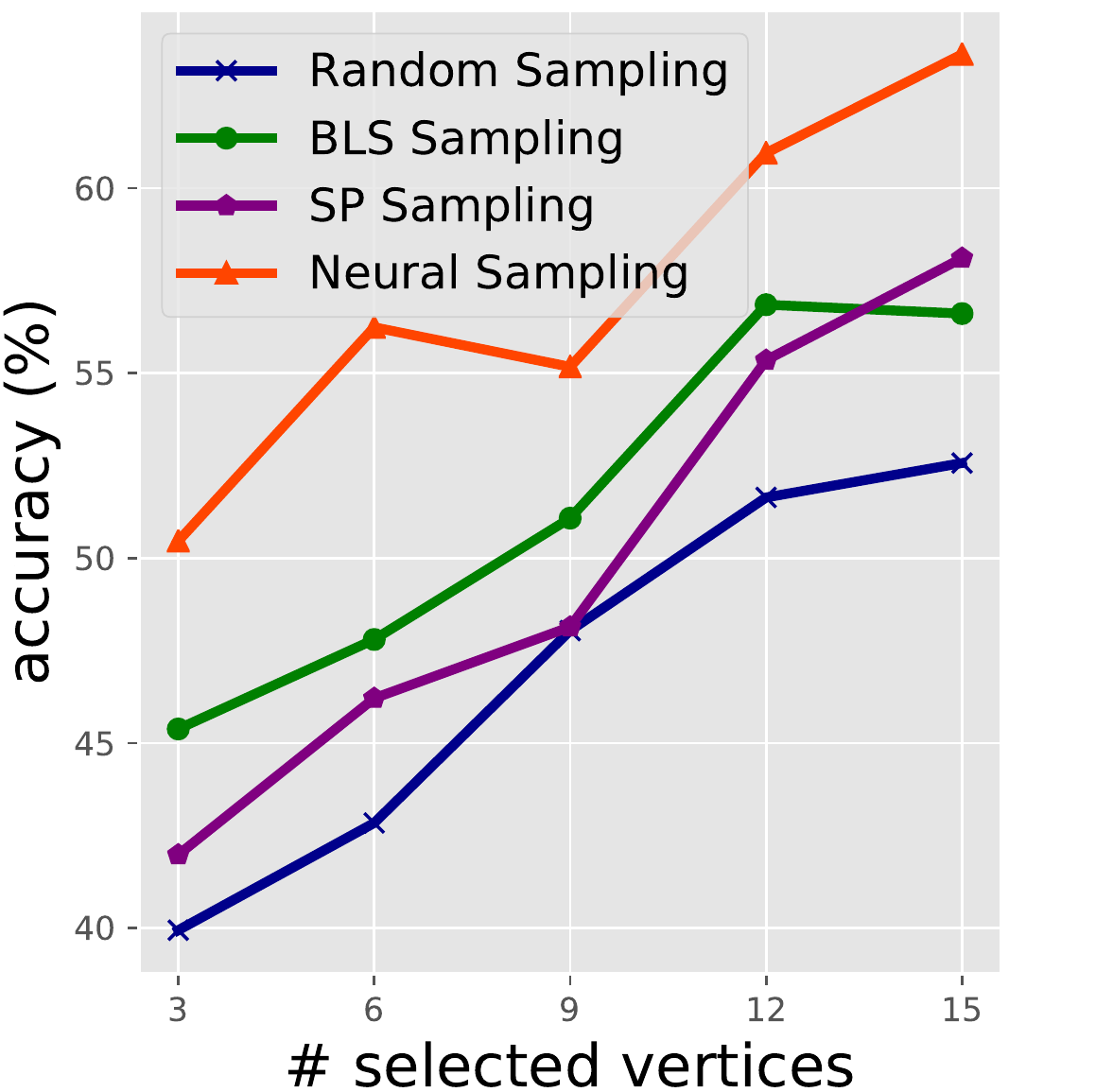}
    \\
    {\small (a) Cora.} &  
    {\small (b) Citeseer.} & 
    {\small (c) Pubmed.}
  \end{tabular}
\end{center}
\vspace{-10pt}
\caption{\small Semi-supervised vertex classification accuracy as a function of the number of labeled vertices. The labeled vertices are selected by using various sampling methods. The proposed graph neural sampling module (red) significantly outperforms random sampling (blue) and two analytical sampling methods: bandlimited-space (BLS) sampling (green) and spectral-proxy (SP) sampling (purple).}
\vspace{-2pt}
\label{fig:active_sampling}
\end{figure*}

\begin{table*}[t]
  \centering
  \small
  \caption{Vertex classification accuracies (\%) of various methods, where `full-sup.' and `semi-sup.' denote the full-supervised and semi-supervised settings, respectively. The proposed GXN improves the classification accuracy by $1.15\%$ in average.}
  \setlength{\tabcolsep}{3.60mm}{
  \begin{tabular}{c|cc|cc|cc}
  \specialrule{0.08em}{0pt}{1pt}
  Dataset & \multicolumn{2}{c|}{Cora} & \multicolumn{2}{c|}{Citeseer} & \multicolumn{2}{c}{Pubmed} \\
  \# Vertices (Classes) & \multicolumn{2}{c|}{2708 (7)} & \multicolumn{2}{c|}{3327 (6)} & \multicolumn{2}{c}{19717 (3)} \\
  Supervision & full-sup. & semi-sup. & full-sup. & semi-sup. & full-sup. & semi-sup.\\

  \specialrule{0.05em}{1pt}{1pt}
  DeepWalk~\cite{PerozziAS:14} & 78.4 $\pm$ 1.7 & 67.2 $\pm$ 2.0 & 68.5 $\pm$ 1.8 & 43.2 $\pm$ 1.6 & 79.8 $\pm$ 1.1 & 65.3 $\pm$ 1.1 \\
  ChebNet~\cite{DefferrardBV:16} & 86.4 $\pm$ 0.5 & 81.2 $\pm$ 0.5 & 78.9 $\pm$ 0.4 & 69.8 $\pm$ 0.5 & 88.7 $\pm$ 0.3 & 74.4 $\pm$ 0.4 \\
  GCN~\cite{KipfW:17} & 86.6 $\pm$ 0.4 & 81.5 $\pm$ 0.5 & 79.3 $\pm$ 0.5 & 70.3 $\pm$ 0.5 & 90.2 $\pm$ 0.3 & 79.0 $\pm$ 0.3\\
  GAT~\cite{VelivckovicCCRLB:18} & 87.8 $\pm$ 0.7 & 83.0 $\pm$ 0.7 & 80.2 $\pm$ 0.6 & 73.5 $\pm$ 0.7 & 90.6 $\pm$ 0.4 & 79.0 $\pm$ 0.3 \\
  FastGCN~\cite{ChenMX:18} & 85.0 $\pm$ 0.8 & 80.8 $\pm$ 1.0 & 77.6 $\pm$ 0.8 & 69.4 $\pm$ 0.8 & 88.0 $\pm$ 0.6 & 78.5 $\pm$ 0.7 \\
  ASGCN~\cite{HuangZRH:18} & 87.4 $\pm$ 0.3 & - & 79.6 $\pm$ 0.2 & - & 90.6 $\pm$ 0.3 & -\\
  Graph U-Net~\cite{GaoJ:19} & - & 84.4 & - & 73.2 & - & 79.6 \\
  \specialrule{0.05em}{0pt}{1pt}
  GXN & {\bf 88.9 $\pm$ 0.4} & {\bf 85.1 $\pm$ 0.6} & {\bf 80.9 $\pm$ 0.4} & {\bf 74.8 $\pm$ 0.4} & {\bf 91.8 $\pm$ 0.3} & {\bf 80.2 $\pm$ 0.3} \\
  GXN (noCross) & 87.3 $\pm$ 0.4 & 83.2 $\pm$ 0.5 & 79.5 $\pm$ 0.4 & 73.7 $\pm$ 0.3 & 91.1 $\pm$ 0.2 & 79.6 $\pm$ 0.3 \\
  \specialrule{0.08em}{1pt}{1pt}
  \end{tabular}}
  \label{tab:node_classification}
\end{table*}

\begin{table*}[t]
  \centering
  \small
  \caption{Graph classification accuracies (\%) of various methods on multiple datasets. 
  Under our GXN framework, we compare various neural-network-based graph pooling methods.  GXN (gPool), GXN (SAGPool) and GXN (AttPool) denote GXN with gPool~\cite{GaoJ:19}, SAGPool~\cite{LeeLK:19} and AttPool~\cite{HuangLLLL:19}. Under the GXN framework, we further compare various feature-crossing methods, including fusion of coarse-to-fine ($\uparrow$), fine-to-coarse ($\downarrow$), no feature-crossing (noCross), and feature-crossing at early, late and all layers of networks. The proposed GXN improves the classification accuracy by $1.30\%$ in average.}
  \setlength{\tabcolsep}{2.92mm}{
  \begin{tabular}{ccccccc}
  \specialrule{0.08em}{0pt}{1pt}
  Dataset & IMDB-B & IMDB-M & COLLAB & D\&D & PROTEINS & ENZYMES \\
  \# Graphs (Classes) & 1000 (2) & 1500 (3) & 5000 (3) & 1178 (2) & 1113 (2) & 600 (6) \\
  Avg. \# Vertices & 19.77 & 13.00 & 74.49 & 284.32 & 39.06 & 32.63 \\
  \specialrule{0.05em}{1pt}{1pt}
  PatchySAN~\cite{NiepertAK:16} & 76.27 $\pm$ 2.6 & 69.70 $\pm$ 2.2 & 43.33 $\pm$ 2.8 & 72.60 $\pm$ 2.2 & 75.00 $\pm$ 2.8 & - \\
  ECC~\cite{SimonovskyK:17} & - & - & 67.79 & 72.54 & 72.65 & 53.50 \\
  Set2Set~\cite{GilmerSRVD:17} & - & - & 71.75 & 78.12 & 74.29 & 60.15 \\
  DGCNN~\cite{ZhangCNC:18} & 70.00 $\pm$ 0.9 & 47.83 $\pm$ 0.9 & 73.76 $\pm$ 0.5 & 79.37 $\pm$ 0.9 & 73.68 $\pm$ 0.9 & - \\
  DiffPool~\cite{YingYMRHL:18} & 70.40 & 47.83 & 75.84 & 80.64 & 76.25 &  {62.53} \\
  Graph U-Net~\cite{GaoJ:19} & 72.10 & 48.33 & 77.56 & 82.43 & 77.68 & 58.57 \\
  SAGPool~\cite{LeeLK:19} & 72.80 & 49.43 & 78.52 & 82.84 & 78.28 & 60.23 \\
  AttPool~\cite{HuangLLLL:19} & 73.60 & 50.67 & 77.04 & 79.20 & 76.50 & 59.76 \\
  StructPool~\cite{YuanJ:20} & {74.70} & {52.47} & 74.22 & 84.19 & {80.36} & {\bf 63.83} \\
  \specialrule{0.05em}{0pt}{1pt}
  GXN & {\bf 77.30 $\pm$ 0.8} & {\bf 54.57 $\pm$ 0.9} & {\bf 80.62 $\pm$ 0.8} &  {\bf 84.26 $\pm$ 1.3} & {\bf 80.38 $\pm$ 1.2} & 60.43 $\pm$ 1.0 \\
  GXN (gPool) & 76.40 $\pm$ 1.0 & 53.16 $\pm$ 0.6 & 79.85 $\pm$ 1.1 & 83.44 $\pm$ 1.4 & 78.74 $\pm$ 0.8 & 59.74 $\pm$ 1.3 \\
  GXN (SAGPool) & 76.90 $\pm$ 0.7 & 52.74 $\pm$ 0.8 & 80.28 $\pm$ 0.8 & 84.14 $\pm$ 1.3 & 79.58 $\pm$ 1.1 & 59.87 $\pm$ 1.1 \\
  GXN (AttPool) & 76.85 $\pm$ 0.9 & 53.62 $\pm$ 0.9 & 80.37 $\pm$ 0.9 & 84.07 $\pm$ 1.0 & 79.09 $\pm$ 1.3 & 59.45 $\pm$ 1.0 \\
  \specialrule{0.05em}{0pt}{1pt}
  GXN ($\uparrow$) & 77.10 $\pm$ 0.6 & 54.22 $\pm$ 1.0 & 80.11 $\pm$ 0.8 & 84.13 $\pm$ 1.0 & 79.87 $\pm$ 0.8 & 59.48 $\pm$ 0.8 \\
  GXN ($\downarrow$) & 76.80 $\pm$ 1.1 & 54.08 $\pm$ 0.7 & 80.28 $\pm$ 1.0 & 83.61 $\pm$ 1.2 & 79.64 $\pm$ 1.2 & 58.95 $\pm$ 1.3 \\
  GXN (noCross) & 74.80 $\pm$ 1.1 &  52.68 $\pm$ 0.9 & 79.94 $\pm$ 0.7 &  83.64 $\pm$ 0.9 &  79.26 $\pm$ 0.9 & 59.37 $\pm$ 1.2 \\
  GXN (early) & {77.10} $\pm$ 0.6 & 53.83 $\pm$ 0.6 & 80.18 $\pm$ 0.8 &  84.24 $\pm$ 1.0 &  80.30 $\pm$ 1.0 & 60.43 $\pm$ 1.0 \\
  GXN (late) & 76.30 $\pm$ 0.9 & 54.12 $\pm$ 1.0 & 79.88 $\pm$ 1.1 & 83.85 $\pm$ 1.5 & 80.03 $\pm$ 1.2 & 59.84 $\pm$ 0.9 \\
  \specialrule{0.08em}{1pt}{1pt}
  \end{tabular}}
  \label{tab:graph_classification}
\end{table*}

\section{Applications} 
\label{sec:application}
In this section, we present three applications: active-sampling-based semi-supervised classification, vertex classification and graph classification. The first application validates the quality of the selected vertices obtained by the graph neural sampling module. The second and third applications show the superiority of the proposed graph cross network.

\subsection{Active-sampling-based semi-supervised classification}
The task is to classify each vertex to a predefined category. Here we are allowed to actively query the category labels associated with a few selected vertices as our training data and then classify all the rest vertices in a semi-supervised paradigm. We compare various sampling methods followed by the same classifier. The classification accuracy thus reflects the amount of information carried by the selected vertices. The goal of this experiment is to validate the effectiveness of the proposed graph neural sampling module.

\mypar{Datasets}
We use three classical citation networks: Cora, Citeseer and Pubmed~\cite{KipfW:17}, whose vertices are articles and edges are references. Cora has $2,708$ vertices with $7$ predefined vertex categories, Citeseer has $3,327$ vertices with $6$ predefined vertex categories and Pubmed has $19,717$ vertices with $3$ predefined vertex categories. Each dataset has multiple binary vertex features. 

\mypar{Experimental setup}
We consider four sampling methods to actively select a few vertices: random sampling, which selects each node randomly and uniformly, bandlimited-space (BLS) sampling~\cite{ChenVSK:15}, spectral-proxy (SP) sampling ~\cite{AnisCO:16} and the proposed graph neural sampling module, which uses both vertex features and graph structures. Once the training samples are selected, we then use the same standard graph convolutional network~\cite{KipfW:17} as the semi-supervised classifier to classify the rest vertices. 

\mypar{Results}
Figure~\ref{fig:active_sampling} shows the classification accuracy as a function of the number of selected vertices in three datasets. The $x$-axis is the number of selected vertices and the $y$-axis is the classification accuracy. We expect with more selected vertices, the classification accuracy is increasing. We see that across three different datasets, The proposed graph neural sampling module (in red) significantly outperforms the other methods. For example, in Cora, given $7$ selected vertices, the gap between the proposed graph neural sampling module and bandlimited-space sampling is more than $10\%$. At the same time, two analytical sampling methods, bandlimited-space sampling (in green) and spectral-proxy sampling (in purple), consistently outperform random sampling (in blue). The intuition is that both analytical sampling methods assume smooth graph signal models, which are beneficial, but cannot perfectly fit arbitrary datasets; while the proposed graph neural sampling module combines information from both vertex features and graph structures to capture the underlying implicit graph signal model, and is able to adaptively select informative vertices in each dataset.

\subsection{Vertex classification}
The task is to classify each vertex to a predefined category under both full-supervised and semi-supervised settings. The goal of this experiment is to validate the effectiveness of the proposed multiscale graph neural network, GXN.

\mypar{Datasets}
We use three standard citation networks: Cora, Citeseer and Pubmed~\cite{KipfW:17}. We perform both full-supervised and semi-supervised vertex classification. For full-supervised classification, we label all the vertices in training sets for model training; for semi-supervised, we only label a few vertices (around 7\% on average) in training sets. We use the default separations of training/validation/test subsets~\cite{KipfW:17}.

\mypar{Experimental setup}
We consider three scales in the proposed GXN, which preserve $100\%$, $90\%$ and $70\%$ vertices from the original scales, respectively. For both input and readout layers, we use one graph convolution layer~\cite{KipfW:17}; for multiscale feature extraction, we use two graph convolution layers followed by ReLUs at each scale and feature-crossing layers between any two consecutive scales at any layers. For the involved graph neural sampling module in GXN, we modify $C(\Omega)$ by preserving only the first term to improve the efficiency of solving problem~\eqref{eq:problem}. In this way, each vertex contributes the vertex set independently. The optimal solution is to select top-$K$ vertices. The hidden feature is 128-dimension across the network. In the loss function~\eqref{eq:vertex_classification_loss}, the hyperparameter $\alpha$ decays from $2$ to $0$ during training, where the graph neural sampling module needs fast convergence for vertex selection; and the model gradually focuses more on tasks based on the effective sampling. We use Adam optimizer~\cite{Goodfellow:2016} and the learining rates range from $0.0001$ to $0.001$ for different datasets.

\mypar{Results}
We compare the proposed GXN to state-of-the-art methods: DeepWalk~\cite{PerozziAS:14}, GCN~\cite{KipfW:17}, GraphSAGE~\cite{HamiltonYL:17}, FastGCN~\cite{ChenMX:18}, ASGCN~\cite{HuangZRH:18}, and Graph U-Net~\cite{GaoJ:19} for vertex classification. We reproduce these methods for both full-supervised and semi-supervised learning based on their official codes. 
Table~\ref{tab:node_classification} compares the vertex classification accuracies of various methods. Under both full-supervised and semi-supervised settings, the proposed GXN achieves higher average accuracy by $1.15\%$.  To be specific,  GXN(noCross) means a degraded GXN without any feature-crossing layer. Overall, we see that the feature-crossing layers improves the accuracies by $1.10\%$ on average. Introducing more connections across graph scales improves the classification performances.

\subsection{Graph classification}
The task is to classify an entire graph to a predefined category. The goal of this experiment is again, to validate the effectiveness of the proposed GXN.

\textbf{Datasets.} We use social network datasets: IMDB-B, IMDB-M and COLLAB~\cite{YanardagV:15}, and bioinformatic datasets: D\&D~\cite{DobsonD:03}, PROTEINS~\cite{FeragenKPBB:13}, and ENZYMES~\cite{BorgwardtSSVSK:05}. Table~\ref{tab:graph_classification} shows the dataset information. Note that no vertex feature is provided in three social network datasets, and we use one-hot vectors to encode the vertex degrees as vertex features, explicitly utilizing some structural information. We use the same dataset separation as in~\cite{GaoJ:19}, perform 10-fold cross-validation, and show the average accuracy for evaluation.

\mypar{Experimental setup}
We consider the same setting used in the task of the vertex classification for the proposed GXN.  The only difference is that after the readout layers, we unify various graph embeddings to the same dimension by using the same aggregation method adopted in DGCNN~\cite{ZhangCNC:18}, AttPool~\cite{HuangLLLL:19} and Graph U-Net~\cite{GaoJ:19}.

\mypar{Results} We compare the proposed GXN to other GNN-based methods, including 
PatchySAN~\cite{NiepertAK:16}, ECC~\cite{SimonovskyK:17}, 
Set2Set~\cite{GilmerSRVD:17}, 
DGCNN~\cite{ZhangCNC:18}, 
DiffPool~\cite{YingYMRHL:18}, 
Graph U-Net~\cite{GaoJ:19}, 
SAGPool~\cite{LeeLK:19}, 
AttPool~\cite{HuangLLLL:19}, 
and StructPool~\cite{YuanJ:20},
where most of them performed multiscale graph feature learning. 
Additionally, we design several variants of GXN: 1) to test the superiority of the proposed graph neural sampling module, we apply gPool~\cite{GaoJ:19}, SAGPool~\cite{LeeLK:19} and AttPool~\cite{HuangLLLL:19} in the same architecture of GXN, denoted as GXN (gPool), GXN (SAGPool) and GXN (AttPool), respectively; 2) we investigate different feature-crossing mechanism, including various crossing directions and crossing positions. 

Table~\ref{tab:graph_classification} compares the accuracies of various methods for graph classification. We see that our model outperforms the state-of-the-art methods on $5$ out of $6$ datasets, achieving an improvement by $1.30\%$ on average accuracies. Besides, GXN with the proposed graph neural sampling module lead to better classification performance than GXNs with the other neural-network-based graph pooling methods, such as gPool, SAGPool and AttPool. Further, we compare GXN with various feature-crossing methods, including fusion of coarse-to-fine ($\uparrow$), fine-to-coarse ($\downarrow$), no feature-crossing (noCross), and feature-crossing at the first layer (early) and at the late layer (late). We see that GXN with more crossing connections performs consistently better, which is consistent with what we see in the task of vertex classification.

\section{Conclusions}
\label{sec:conclusions}
We propose interpretable graph neural networks for sampling and recovery of graph signals. Through neural estimation of mutual information between vertex and neighborhood features, the proposed neural sampling module is optimized to select those vertices that maximally reflect their corresponding neighborhoods. The proposed neural recovery module is based on algorithm  unrolling, which transforms each iteration of an analytical recovery algorithm to a network layer. Compared to previous analytical sampling and recovery, the proposed methods are able to learn a variety of graph signal models from data by leveraging the learning ability of neural networks. Meanwhile, the proposed methods exploit the graph-related properties and provide interpretability. Based  on  the  proposed graph neural  sampling  and  recovery  modules,  we  propose  a  new  multiscale  graph  neural  network, which  is  a trainable counterpart  of  a  multiscale  graph  filter bank. It introduce a new feature-crossing layer,  allowing intermediate features from multiple graph scales to communicate and merge.  

In the experiments, we illustrate the effects of the proposed neural sampling and recovery modules and find that the proposed modules can flexibly adapt to various graph structures and graph signals. In the task of active-sampling-based semi-supervised learning, the proposed neural sampling improves the classification accuracy over $10\%$ in Cora dataset. We further validate the proposed multiscale graph neural network on several standard datasets for both vertex classification and graph classification. The result shows that the  proposed  method  consistently improves the classification accuracy.

\bibliographystyle{IEEEbib}
\bibliography{refs}

\begin{thebibliography}{10}

\bibitem{Jackson:08}
M.~Jackson,
\newblock {\em Social and Economic Networks},
\newblock Princeton University Press, 2008.

\bibitem{Newman:10}
M.~Newman,
\newblock {\em Networks: An Introduction},
\newblock Oxford University Press, 2010.

\bibitem{OrtegaFKMV:18}
A.~Ortega, P.~Frossard, J.~Kova{\v c}evi{\'c}, J.~M.~F. Moura, and
  P.~Vandergheynst,
\newblock ``Graph signal processing: Overview, challenges, and applications,''
\newblock {\em Proceedings of the {IEEE}}, vol. 106, no. 5, pp. 808--828, 2018.

\bibitem{BronsteinBLSV:17}
M.~M. Bronstein, J.~Bruna, Y.~LeCun, A.~Szlam, and P.~Vandergheynst,
\newblock ``Geometric deep learning: Going beyond {E}uclidean data,''
\newblock {\em {IEEE} Signal Process. Mag.}, vol. 34, no. 4, pp. 18--42, 2017.

\bibitem{ShumanNFOV:13}
D.~I. Shuman, S.~K. Narang, P.~Frossard, A.~Ortega, and P.~Vandergheynst,
\newblock ``The emerging field of signal processing on graphs: {E}xtending
  high-dimensional data analysis to networks and other irregular domains,''
\newblock {\em IEEE Signal Process. Mag.}, vol. 30, pp. 83--98, May 2013.

\bibitem{SandryhailaM:14}
A.~Sandryhaila and J.~M.~F. Moura,
\newblock ``Big data processing with signal processing on graphs,''
\newblock {\em IEEE Signal Process. Mag.}, vol. 31, no. 5, pp. 80--90, Sept.
  2014.

\bibitem{HammondVG:11}
D.~K. Hammond, P.~Vandergheynst, and R.~Gribonval,
\newblock ``Wavelets on graphs via spectral graph theory,''
\newblock {\em Appl. Comput. Harmon. Anal.}, vol. 30, pp. 129--150, Mar. 2011.

\bibitem{NarangO:12}
S.~K. Narang and A.~Ortega,
\newblock ``Perfect reconstruction two-channel wavelet filter banks for graph
  structured data,''
\newblock {\em IEEE Trans. Signal Process.}, vol. 60, pp. 2786--2799, June
  2012.

\bibitem{ShumanFV:16}
D.~I. Shuman, M.~Javad Faraji, and P.~Vandergheynst,
\newblock ``A multiscale pyramid transform for graph signals,''
\newblock {\em {IEEE} Trans. Signal Process.}, vol. 64, no. 8, pp. 2119--2134,
  2016.

\bibitem{DongTRF:19}
X.~Dong, D.~Thanou, M.~Rabbat, and P.~Frossard,
\newblock ``Learning graphs from data: {A} signal representation perspective,''
\newblock {\em {IEEE} Signal Process. Mag.}, vol. 36, no. 3, pp. 44--63, 2019.

\bibitem{KipfW:17}
T.~N. Kipf and M.~Welling,
\newblock ``Semi-supervised classification with graph convolutional networks,''
\newblock in {\em {ICLR} 2017, Toulon, France, April 24-26, 2017, Conference
  Track Proceedings}, 2017.

\bibitem{WangSLSBS:19}
Y.~Wang, Y.~Sun, Z.~Liu, S.~E. Sarma, M.~M. Bronstein, and J.~M. Solomon,
\newblock ``Dynamic graph {CNN} for learning on point clouds,''
\newblock {\em {ACM} Trans. Graph.}, vol. 38, no. 5, pp. 146:1--146:12, 2019.

\bibitem{LiCCZWT:19}
M.~Li, S.~Chen, X.~Chen, Y.~Zhang, Y.~Wang, and Q.~Tian,
\newblock ``Actional-structural graph convolutional networks for skeleton-based
  action recognition,''
\newblock in {\em Proceedings of the IEEE Conference on Computer Vision and
  Pattern Recognition}, 2019, pp. 3595--3603.

\bibitem{HuCZG:20}
Y.~Hu, S.~Chen, Y.~Zhang, and X.~Gu,
\newblock ``Collaborative motion prediction via neural motion message
  passing,''
\newblock in {\em Proceedings of the IEEE Conference on Computer Vision and
  Pattern Recognition}, 2020.

\bibitem{ChenVSK:15}
S.~Chen, R.~Varma, A.~Sandryhaila, and J.~Kova{\v c}evi{\'c},
\newblock ``Discrete signal processing on graphs: {S}ampling theory,''
\newblock {\em IEEE Trans. Signal Process.}, vol. 63, no. 24, pp. 6510--6523,
  Dec. 2015.

\bibitem{AnisAO:15}
A.~Anis, A.~Gadde, and A.~Ortega,
\newblock ``Efficient sampling set selection for bandlimited graph signals
  using graph spectral proxies,''
\newblock {\em IEEE Trans. Signal Process.}, vol. 64, pp. 3775--3789, July
  2016.

\bibitem{BaiWCNG:20}
Y.~Bai, F.~Wang, G.~Cheung, Y.~Nakatsukasa, and W.~Gao,
\newblock ``Fast graph sampling set selection using gershgorin disc
  alignment,''
\newblock {\em {IEEE} Trans. Signal Process.}, vol. 68, pp. 2419--2434, 2020.

\bibitem{YingYMRHL:18}
Z.~Ying, J.~You, C.~Morris, X.~Ren, W.~Hamilton, and J.~Leskovec,
\newblock ``Hierarchical graph representation learning with differentiable
  pooling,''
\newblock in {\em Advances in Neural Information Processing Systems (NeurIPS)},
  December 2018.

\bibitem{LeeLK:19}
J.~Lee, I.~Lee, and J.~Kang,
\newblock ``Self-attention graph pooling,''
\newblock in {\em Proceedings of the International Conference on Machine
  Learning (ICML)}, July 2019.

\bibitem{VetterliKG:12}
M.~Vetterli, J.~Kova{\v c}evi{\'c}, and V.~K. Goyal,
\newblock {\em Foundations of Signal Processing},
\newblock Cambridge University Press, Cambridge, 2014,
\newblock http://foundationsofsignalprocessing.org.

\bibitem{WangCG:16}
X.~Wang, J.~Chen, and Y.~Gu,
\newblock ``Local measurement and reconstruction for noisy bandlimited graph
  signals,''
\newblock {\em Signal Process.}, vol. 129, pp. 119--129, 2016.

\bibitem{MarquesSGR:15}
A.~G. Marques, S.~Segarra, G.~Leus, and A.~Ribeiro,
\newblock ``Sampling of graph signals with successive local aggregations,''
\newblock {\em IEEE Trans. Signal Process.}, vol. 64, pp. 1832--1843, Dec.
  2015.

\bibitem{AnisCO:16}
A.~Anis, P.~A. Chou, and A.~Ortega,
\newblock ``Compression of dynamic 3d point clouds using subdivisional meshes
  and graph wavelet transforms,''
\newblock in {\em ICASSP}, Shanghai, China, Mar. 2016, pp. 6360--6364.

\bibitem{TanakaE:20}
Y.~Tanaka and Y.~C. Eldar,
\newblock ``Generalized sampling on graphs with subspace and smoothness
  priors,''
\newblock {\em {IEEE} Trans. Signal Process.}, vol. 68, pp. 2272--2286, 2020.

\bibitem{SakiyamaTTO:16}
A.~Sakiyama, Y.~Tanaka, T.~Tanaka, and A.~Ortega,
\newblock ``Efficient sensor position selection using graph signal sampling
  theory,''
\newblock in {\em {ICASSP} 2016, Shanghai, China, March 20-25, 2016}, 2016, pp.
  6225--6229.

\bibitem{GaddeAO:14}
A.~Gadde, A.~Anis, and A.~Ortega,
\newblock ``Active semi-supervised learning using sampling theory for graph
  signals,''
\newblock in {\em Proc. ACM Int. Conf. Knowl. Discovery Data Mining}, New York,
  NY, 2014, pp. 492--501.

\bibitem{ChenYFK:16}
S.~Chen, Y.~Yang, C.~Faloutsos, and J.~Kovacevic,
\newblock ``Monitoring manhattan's traffic at 5 intersections?,''
\newblock in {\em 2016 {IEEE} Global Conference on Signal and Information
  Processing, GlobalSIP 2016, Washington, DC, USA, December 7-9, 2016}. 2016,
  pp. 1270--1274, {IEEE}.

\bibitem{ChenTFVK:18}
S.~Chen, D.~Tian, C.~Feng, A.~Vetro, and J.~Kovacevic,
\newblock ``Fast resampling of three-dimensional point clouds via graphs,''
\newblock {\em {IEEE} Trans. Signal Process.}, vol. 66, no. 3, pp. 666--681,
  2018.

\bibitem{RonnebergerFB:15}
O.~Ronneberger, P.~Fischer, and T.~Brox,
\newblock ``U-net: Convolutional networks for biomedical image segmentation,''
\newblock in {\em {MICCAI}}. 2015, vol. 9351, pp. 234--241, Springer.

\bibitem{VelivckovicCCRLB:18}
P.~Veli{\v{c}}kovi{\'c}, G.~Cucurull, A.~Casanova, A.~Romero, P.~Lio, and
  Y.~Bengio,
\newblock ``Graph attention networks,''
\newblock in {\em International Conference on Learning Representations,
  {ICLR}}, 2018.

\bibitem{GaoJ:19}
H.~Gao and S.~Ji,
\newblock ``Graph {U}-{N}ets,''
\newblock in {\em Proceedings of the International Conference on Machine
  Learning (ICML)}, July 2019.

\bibitem{DengZWZF:20}
C.~Deng, Z.~Zhao, Y.~Wang, Z.~Zhang, and Z.~Feng,
\newblock ``Graphzoom: A multi-level spectral approach for accurate and
  scalable graph embedding,''
\newblock in {\em International Conference on Learning Representations (ICLR)},
  April 2020.

\bibitem{SandryhailaM:13}
A.~Sandryhaila and J.~M.~F. Moura,
\newblock ``Discrete signal processing on graphs,''
\newblock {\em IEEE Trans. Signal Process.}, vol. 61, no. 7, pp. 1644--1656,
  Apr. 2013.

\bibitem{ChenVSK:15c}
S.~Chen, R.~Varma, A.~Singh, and J.~Kova{\v c}evi{\'c},
\newblock ``Signal recovery on graphs: {F}undamental limits of sampling
  strategies,''
\newblock {\em IEEE Trans, Signal and Inform. Process. over Networks}, 2016,
\newblock Submitted.

\bibitem{Eldar:15}
Y.~C. Eldar,
\newblock {\em Sampling theory: Beyond bandlimited systems},
\newblock Cambridge University Press, 2015.

\bibitem{ChenVSK:15a}
S.~Chen, R.~Varma, A.~Singh, and J.~Kova{\v c}evi{\'c},
\newblock ``Signal recovery on graphs: {R}andom versus experimentally designed
  sampling,''
\newblock in {\em Proc. Sampling Theory Appl.}, Washington, DC, May 2015, pp.
  337--341.

\bibitem{PuyTGV:15}
G.~Puy, N.~Tremblay, R.~Gribonval, and P.~Vandergheynst,
\newblock ``Random sampling of bandlimited signals on graphs,''
\newblock {\em CoRR}, vol. abs/1511.05118, 2015.

\bibitem{ChenSMK:14}
S.~Chen, A.~Sandryhaila, J.~M.~F. Moura, and J.~Kova{\v c}evi{\'c},
\newblock ``Signal recovery on graphs: {Variation} minimization,''
\newblock {\em IEEE Trans. Signal Process.}, vol. 63, no. 17, pp. 4609--4624,
  Sept. 2015.

\bibitem{LorenzoBB:20}
P.~Di Lorenzo, S.~Barbarossa, and P.~Banelli,
\newblock ``Sampling and recovery of graph signals,''
\newblock {\em CoRR}, vol. arxiv/abs/1712.09310, 2020.

\bibitem{TanakaEOC:20}
Y.~Tanaka, Y.~C. Eldar, A.~Ortega, and G.~Cheung,
\newblock ``Sampling on graphs: From theory to applications,''
\newblock {\em CoRR}, vol. abs/2003.03957, 2020.

\bibitem{Goodfellow:2016}
I.~Goodfellow, Y.~Bengio, and A.~Courville,
\newblock {\em Deep Learning},
\newblock MIT Press, 2016,
\newblock \url{http://www.deeplearningbook.org}.

\bibitem{BelghaziBROBCH:18}
M.~I. Belghazi, A.~Baratin, S.~Rajeshwar, S.~Ozair, Y.~Bengio, A.~Courville,
  and D.~Hjelm,
\newblock ``Mutual information neural estimation,''
\newblock in {\em Proceedings of the International Conference on Machine
  Learning (ICML)}, July 2018.

\bibitem{NowozinCT:16}
S.~Nowozin, B.~Cseke, and R.~Tomioka,
\newblock ``f-gan: Training generative neural samplers using variational
  divergence minimization,''
\newblock in {\em Advances in Neural Information Processing Systems (NeurIPS)},
  December 2016.

\bibitem{VelickovicFHLBH:19}
P.~Velickovic, W.~Fedus, W.~L. Hamilton, P.~Li{\`{o}}, Y.~Bengio, and R.~Devon
  Hjelm,
\newblock ``Deep graph infomax,''
\newblock in {\em 7th International Conference on Learning Representations,
  {ICLR} 2019, New Orleans, LA, USA, May 6-9, 2019}, 2019.

\bibitem{HjelmFLGBTB:19}
R.~Devon Hjelm, A.~Fedorov, S.~Lavoie-Marchildon, K.~Grewal, P.~Bachman,
  A.~Trischler, and Y.~Bengio,
\newblock ``Learning deep representations by mutual information estimation and
  maximization,''
\newblock in {\em International Conference on Learning Representations {ICLR}},
  April 2019.

\bibitem{ChenHHKK:15}
Y.~Chen, S.~Hamed Hassani, A.~Karbasi, and A.~Krause,
\newblock ``Sequential information maximization: When is greedy
  near-optimal?,''
\newblock in {\em Proceedings of The 28th Conference on Learning Theory, {COLT}
  2015, Paris, France, July 3-6, 2015}. 2015, vol.~40 of {\em {JMLR} Workshop
  and Conference Proceedings}, pp. 338--363, JMLR.org.

\bibitem{MikolovSCCD:13}
T.~Mikolov, I.~Sutskever, K.~Chen, G.~S. Corrado, and J.~Dean,
\newblock ``Distributed representations of words and phrases and their
  compositionality,''
\newblock in {\em 27th Annual Conference on Neural Information Processing
  Systems 2013, Lake Tahoe, Nevada, United States}, 2013, pp. 3111--3119.

\bibitem{ChenVSK:16}
S.~Chen, R.~Varma, A.~Singh, and J.~Kova{\v c}evi{\'c},
\newblock ``Representations of piecewise smooth signals on graphs,''
\newblock in {\em 2016 {IEEE} International Conference on Acoustics, Speech and
  Signal Processing, {ICASSP} 2016, Shanghai, China, March 20-25, 2016}, 2016,
  pp. 6370--6374.

\bibitem{GamaMLR:19}
F.~Gama, A.~G. Marques, G.~Leus, and A.~Ribeiro,
\newblock ``Convolutional neural network architectures for signals supported on
  graphs,''
\newblock {\em {IEEE} Trans. Signal Process.}, vol. 67, no. 4, pp. 1034--1049,
  2019.

\bibitem{WangSST:16}
Y{-}X. Wang, J.~Sharpnack, A.~J. Smola, and R.~J. Tibshirani,
\newblock ``Trend filtering on graphs,''
\newblock {\em J. Mach. Learn. Res.}, vol. 17, pp. 105:1--105:41, 2016.

\bibitem{HuCOA:15}
W.~Hu, G.~Cheung, A.~Ortega, and O.~C. Au,
\newblock ``Multiresolution graph fourier transform for compression of
  piecewise smooth images,''
\newblock {\em {IEEE} Trans. Image Process.}, vol. 24, no. 1, pp. 419--433,
  2015.

\bibitem{DorflerB:13}
Florian D{\"{o}}rfler and Francesco Bullo,
\newblock ``Kron reduction of graphs with applications to electrical
  networks,''
\newblock {\em {IEEE} Trans. Circuits Syst. {I} Regul. Pap.}, vol. 60-I, no. 1,
  pp. 150--163, 2013.

\bibitem{ZhangCNC:18}
M.~Zhang, Z.~Cui, M.~Neumann, and Y.~Chen,
\newblock ``An end-to-end deep learning architecture for graph
  classification,''
\newblock in {\em {AAAI}, New Orleans, Louisiana, USA, February 2-7, 2018}.
  2018, pp. 4438--4445, {AAAI} Press.

\bibitem{ShumanWHV:15}
D.~I. Shuman, C.~Wiesmeyr, N.~Holighaus, and P.~Vandergheynst,
\newblock ``Spectrum-adapted tight graph wavelet and vertex-frequency frames,''
\newblock {\em {IEEE} Trans. Signal Process.}, vol. 63, no. 16, pp. 4223--4235,
  2015.

\bibitem{SakiyamaWTO:19}
A.~Sakiyama, K.~Watanabe, Y.~Tanaka, and A.~Ortega,
\newblock ``Two-channel critically sampled graph filter banks with spectral
  domain sampling,''
\newblock {\em {IEEE} Trans. Signal Process.}, vol. 67, no. 6, pp. 1447--1460,
  2019.

\bibitem{GamaRB:19}
F.~Gama, A.~Ribeiro, and J.~Bruna,
\newblock ``Stability of graph scattering transforms,''
\newblock in {\em Advances in Neural Information Processing Systems 32: Annual
  Conference on Neural Information Processing Systems 2019, NeurIPS 2019, 8-14
  December 2019, Vancouver, BC, Canada}, 2019, pp. 8036--8046.

\bibitem{Abbe18}
E.~Abbe,
\newblock ``Community detection and stochastic block models,''
\newblock {\em Found. Trends Commun. Inf. Theory}, vol. 14, no. 1-2, pp.
  1--162, 2018.

\bibitem{PerozziAS:14}
B.~Perozzi, R.~Al{-}Rfou, and S.~Skiena,
\newblock ``Deepwalk: online learning of social representations,''
\newblock in {\em The 20th {ACM} {SIGKDD} International Conference on Knowledge
  Discovery and Data Mining, {KDD} '14, New York, NY, {USA} - August 24 - 27,
  2014}. 2014, pp. 701--710, {ACM}.

\bibitem{DefferrardBV:16}
M.~Defferrard, X.~Bresson, and P.~Vandergheynst,
\newblock ``Convolutional neural networks on graphs with fast localized
  spectral filtering,''
\newblock in {\em Advances in Neural Information Processing Systems 29: Annual
  Conference on Neural Information Processing Systems 2016, December 5-10,
  2016, Barcelona, Spain}, 2016, pp. 3837--3845.

\bibitem{ChenMX:18}
J.~Chen, T.~Ma, and C.~Xiao,
\newblock ``Fast{GCN}: Fast learning with graph convolutional networks via
  importance sampling,''
\newblock in {\em International Conference on Learning Representations (ICLR)},
  April 2018.

\bibitem{HuangZRH:18}
W.~Huang, T.~Zhang, Y.~Rong, and J.~Huang,
\newblock ``Adaptive sampling towards fast graph representation learning,''
\newblock in {\em Advances in Neural Information Processing Systems (NeurIPS)},
  December 2018.

\bibitem{HuangLLLL:19}
J.~Huang, Z.~Li, N.~Li, S.~Liu, and G.~Li,
\newblock ``Attpool: Towards hierarchical feature representation in graph
  convolutional networks via attention mechanism,''
\newblock in {\em The IEEE International Conference on Computer Vision (ICCV)},
  October 2019.

\bibitem{NiepertAK:16}
M.~Niepert, M.~Ahmed, and K.~Kutzkovl,
\newblock ``Learning convolutional neural networks for graphs,''
\newblock in {\em Proceedings of the International Conference on Machine
  Learning (ICML)}, June 2016.

\bibitem{SimonovskyK:17}
M.~Simonovsky and N.~Komodakis,
\newblock ``Dynamic edge-conditioned filters in convolutional neural networks
  on graphs,''
\newblock in {\em The IEEE Conference on Computer Vision and Pattern
  Recognition (CVPR)}, July 2017, pp. 3693--3702.

\bibitem{GilmerSRVD:17}
J.~Gilmer, S.~S. Schoenholz, P.~F. Riley, O.~Vinyals, and G.~E. Dahl,
\newblock ``Neural message passing for quantum chemistry,''
\newblock in {\em {ICML}}. 2017, vol.~70, pp. 1263--1272, {PMLR}.

\bibitem{YuanJ:20}
Hao Yuan and Shuiwang Ji,
\newblock ``Structpool: Structured graph pooling via conditional random
  fields,''
\newblock in {\em International Conference on Learning Representations (ICLR)},
  April 2020.

\bibitem{HamiltonYL:17}
W.~Hamilton, Z.~Ying, and J.~Leskovec,
\newblock ``Inductive representation learning on large graphs,''
\newblock in {\em Advances in Neural Information Processing Systems}, 2017, pp.
  1024--1034.

\bibitem{YanardagV:15}
P.~Yanardag and S.V.~N. Vishwanathan,
\newblock ``A structural smoothing framework for robust graph comparison,''
\newblock in {\em Advances in Neural Information Processing Systems (NeurIPS)},
  December 2015.

\bibitem{DobsonD:03}
P.~D. Dobson and A.~J. Doig,
\newblock ``Distinguishing enzyme structures from non-enzymes without
  alignments,''
\newblock {\em Journal of Molecular Biology (JMB)}, vol. 330, no. 4, pp.
  771--783, July 2003.

\bibitem{FeragenKPBB:13}
A.~Feragen, N.~Kasenburg, J.~Petersen, M.~de~Bruijne, and K.~Borgwardt,
\newblock ``Scalable kernels for graphs with continuous attributes,''
\newblock in {\em Advances in Neural Information Processing Systems (NeurIPS)},
  December 2013.

\bibitem{BorgwardtSSVSK:05}
K.~M. Borgwardt, C.~Soon Ong, S.~Sch{\" o}nauer, S.~V.~N. Vishwanathan, A.~J.
  Smola, and H.-P. Kriegel,
\newblock ``Protein function prediction via graph kernels,''
\newblock {\em Bioinformatics}, vol. 21, no. 1, pp. i47--i56, March 2005.

\end{thebibliography}
\end{document}